\documentclass{article}

   \usepackage[nonatbib,preprint]{neurips_2024}

\usepackage[numbers]{natbib}
\usepackage[T1]{fontenc}    

\usepackage[colorlinks=true, allcolors=blue]{hyperref}
\usepackage{url}            
\usepackage{booktabs}       
\usepackage{amsfonts}       
\usepackage{nicefrac}       
\usepackage{microtype}      
\usepackage{xcolor}         
\usepackage{dsfont}
\usepackage{amsthm}
\usepackage{amsmath}
\usepackage{enumitem}
\usepackage{algorithm}
\usepackage{microtype}
\usepackage{graphicx}
\usepackage{subfigure}
\usepackage{caption}
\usepackage{dcolumn}
\usepackage{comment}

\newcommand{\bbE}{\mathbb{E}}
\newcommand{\bbR}{\mathbb{R}}
\newcommand{\bbP}{\mathbb{P}}

\newcommand{\ba}{\mathbf{a}}

\newcommand{\calD}{\mathcal{D}}
\newcommand{\calM}{\mathcal{M}}

\newcommand{\calQ}{\mathcal{Q}}

\newtheorem{theorem}{Theorem}

\newtheorem{proposition}{Proposition}

\usepackage{wrapfig}
\usepackage{enumitem}

\title{Better Membership Inference Privacy Measurement through Discrepancy}

\author{%
  Ruihan Wu\thanks{Equal contribution} \\
  University of California, San Diego\\
  \texttt{ruw076@ucsd.edu} \\
  \And
  Pengrun Huang$^*$ \\
  University of California, San Diego\\
  \texttt{peh006@ucsd.edu} \\
  \And
  Kamalika Chaudhuri \\
  University of California, San Diego\\
  \texttt{kamalika@cs.ucsd.edu} \\
}

\begin{document}

\maketitle
\begin{abstract}
Membership Inference Attacks have emerged as a dominant method for empirically measuring privacy leakage from machine learning models. Here, privacy is measured by the {\em{advantage}} or gap between a score or a function computed on the training and the test data. A major barrier to the practical deployment of these attacks is that they do not scale to large well-generalized models -- either the advantage is relatively low, or the attack involves training multiple shadow models which is high compute-intensive.
 In this work, inspired by discrepancy theory, we propose a new empirical privacy metric that is an upper bound on the advantage of a specific family of membership inference attacks.
 We introduce an easy-to-compute approximation CPM to the upper bound. 
We empirically validate CPM is higher than the advantage of most popular existing attacks and because of its light computation, it can be applied to large ImageNet classification models in-the-wild. 
More interestingly, we find the gaps between the new metric and the advantage of existing attacks are larger on advanced models trained with sophisticated training recipes.
Motivated by this empirical result, we propose new membership inference attacks tailored to those advanced models. 
\end{abstract}


\section{Introduction}

Many machine-learning models are now trained on highly sensitive data, such as medical records, browsing history and financial information, and leakage of training data from these models would cause serious concern~\citep{fredrikson2014privacy, papernot2016towards, carlini2019secret}. Consequently, there has been a body of technical literature on how to measure privacy leakage from the training data of machine learning models~\citep{shokri2017membership, yeom2018privacy, zhu2019deep, meehan2024ssl}. One of the most dominant methods for empirically measuring privacy leakage is Membership Inference~\citep{shokri2017membership, yeom2018privacy, jayaraman2020revisiting, watson2021importance, song2021systematic, carlini2022membership, zhu2024uncertainty}, that has been designated as a potential confidentiality violation by government organizations such as NIST (US)  as well as the ICO (UK), and deployed in industry applications, such as the privacy auditing library of Tensorflow~\citep{song2020introducing}. Given a model, a data point, and possibly some auxiliary information, a Membership Inference Attack (MIA) predicts whether the given data point is included in the training data of the model or not. Privacy leakage is measured by the gap between the accuracy of a membership inference attack on the training data and the test data of a model; this gap is called the advantage and higher advantage indicates more privacy leakage.



The current literature on MIA falls into two main categories. The first one is score-based MIA, which is motivated by the idea that certain scoring functions computed on a model and a data point, such as the cross-entropy loss, are quite different on average when evaluated on training and test data points. This idea has led to a proliferation of a number of score functions~\citep{shokri2017membership, yeom2018privacy, song2021systematic, rezaei2021difficulty}; this class of MIAs are computationally efficient, but tend to have a low advantage on large well-generalized models. The second category also uses scoring functions, but adjusts them per data point using ``shadow models''~\citep{shokri2017membership, watson2021importance, carlini2022membership}, which are models very similar to the model in question, but trained on a different auxiliary dataset. These MIAs require the attacker to train multiple shadow models which makes them computationally infeasible for large-scale models.



In this work, inspired by discrepancy theory, we propose a new metric for empirically measuring privacy leakage from models. We observe that the discrepancy between the training and test data with respect to a class of sets $\mathcal{Q}$ is an upper bound on the advantage of any score-based MIA whose discriminative set lies in $\mathcal{Q}$. 
An advantage upper bound is stronger than the advantage of any single MIA to ensure a model is privacy-preserving -- the advantage of any MIA in this family would not exceed this upper bound.   
Thus, we propose the discrepancy with respect to all convex sets in the probability space of a neural network as an empirical privacy metric.  We prove that this metric is a upper bound on the advantage of four popular score-based MIAs -- entropy, maximum-softmax-probability, cross-entropy and modified entropy~\citep{song2021systematic} --and is hence at least as strong as either of them. 
Additionally, our numerical experiments show that this metric has discriminative power, and is able to distinguish between a large number of models. 
Finally, even though the exact computation of our metric may be hard, we propose a new algorithm for approximating it using a surrogate loss function -- a metric that we call CPM, and we show the value of CPM is reachable by some score-based MIA.
Comparing our new metric to the two categories of existing MIAs, it is stronger than the popular score-based MIAs and more computational-feasible than the MIAs leveraging ``shadow models".


We then extensively evaluate CPM by comparing it with score-based MIAs on various models and several datasets, where we follow the setups in the MIA literature, as well as out-of-the-box ImageNet pre-trained models released by the PyTorch Torchvision library~\citep{paszke2019pytorch}. We observe that existing MIAs are upper-bounded by the CPM, which supports that it is a stronger privacy metric. Interestingly, we find the gap between CPM and existing score-based MIAs is small for standard models trained with cross-entropy loss, but considerably larger for models trained with more sophisticated generalization methods or an MIA defense. 
Because CPM is reachable by some score-based MIA,
this suggests that the design of existing scoring-based MIAs may be overfitting to standard models, and other better scoring functions may be needed to measure membership inference properly in the more modern models. This is also corroborated by our findings on the Pytorch ImageNet pre-trained models, where CPM outperforms the baselines significantly in the Resnetv2 models that use a complicated training recipe, and not as much in the simpler Resnetv1 models. 



A natural question suggested by these experimental results is whether there are scoring functions that perform better on models trained in a more sophisticated way. 
This is an interesting question; for example, prior work~\citep{duan2024membership} has shown that currents MIAs do not work for really large language models.
One possible solution is to design training-procedure aware scores.
Motivated from this, for the model trained by the generalization technique MixUp~\citep{zhang2018mixup} or the MIA defense algorithm RelaxLoss~\citep{chen2022relaxloss}, we propose two new score-based MIAs, the MixUp score and the RelaxLoss score, that mimic their training procedure respectively. We empirically observe that the advantage of the training procedure aligned MIA score is the highest.
This suggests that a plausible reason why MIAs do not work as well on really large models might be the use of incorrect scores; we leave the design of more training-aware MIA scores for these modern models for future work.

\section{Preliminaries}
\label{sec:preliminary}
\textbf{Membership inference attack (MIA; \citep{shokri2017membership})} is a privacy attack where the goal is to predict whether a specific data point is included in the training data. 
This membership information can be sensitive -- for example, 
the membership to a medical dataset indicates whether a person has a medical record or not.
These attacks have been studied for a variety of machine learning models including classification~\citep{yeom2018privacy}, generative models~\citep{chen2020gan}, multi-modal models~\citep{ko2023practical} and large language models~\citep{wen2022canary}.


Suppose we have a model $f$ that is trained on a training dataset $S$ drawn from an underlying data distribution $\calD$. The input to an MIA $m$ is a data point $z = (x, y)$ and a trained model $f$, and the output is a $0/1$ value. $m(z, f) = 1$ means that the MIA predicts that the data point $z$ is in the training set $S$ of $f$. 
The advantage of an MIA $m$ w.r.t. the model $f$, training data $S$, and data distribution $\mathcal{D}$ is defined as the difference between how frequently $m$ predicts $1$ on a point in the training set, and how frequently it predicts $1$ on points drawn from an independent test set:
\begin{equation}
    {\rm Adv}(m; f, S, \mathcal{D}) := \bbP_{z\sim S}(m(z, f) = 1) - \bbP_{z\sim \calD}(m(z,f) = 1)\in[-1, 1].
    \label{eq:adv}
\end{equation}
Observe that the advantage is an empirical measure of privacy; if it is high, then it is easier to distinguish between training and test data points, which may, in turn, lead to the leakage of other private training data information even beyond the membership inference, such as attribute inference attack~\citep{yeom2018privacy}. We also observe that we need both training and some test data to calculate the advantage and evaluate how private a model is. 


Finally, we note that contrary to some of the literature~\citep{yeom2018privacy, nasr2021adversary}, our definition of advantage does not require multiple runs of training; this allows us to scale to large datasets and models where training multiple models for the purpose of evaluating privacy is too expensive. 



\textbf{Existing MIA literature} has two classes of MIA. The first, which we call score-based MIA, is motivated by the intuition that certain functions of a model $f$, such as training loss, are lower for training data points than test on average. Accordingly, they use a scoring function $h(z, f)\in \bbR$~\citep{shokri2017membership, yeom2018privacy}, where $z$ is the input data and a threshold to determine membership -- specifically,
\begin{equation}
\label{eq:thr_att}
m_{h, \tau}(z, f)=\mathds{1}\left[h(z, f)< \tau\right].
\end{equation}

In addition to the training loss, prior work has proposed several probability-based scoring functions for classification models. Suppose $f(x) \in\Delta_{C-1}$ is the softmax vector of probabilities output by a $C$-class classification model, and suppose the label $y$ is in one-hot format. Then, 
we summarize some popular probability-based scoring functions in the literature in Table~\ref{tab:mia_score}.
The gradient-based scoring function is another popular choice~\cite{rezaei2021difficulty}, which is the norm of the gradient of $x$ or the gradient of parameters in $f$ w.r.t. the loss.

While score-based MIA attacks are computationally efficient, they may have lower advantage; more recent work~\citep{shokri2017membership, watson2021importance, carlini2022membership} has sought to improve the advantage of the MIAs by leveraging ``shadow models'' -- which are essentially similar models (to the input model $f$) trained on auxiliary data drawn from the same distribution. 
With these additional models, one can design a more elaborate attack: instead of sharing one threshold $\tau$ for all data $z$ as defined in Equation~\ref{eq:thr_att}, one can now design a data-dependent threshold $\tau(z)$ or even a more complicated data-dependent decision boundary for $h(z, f)$ than just a threshold.
\citet{watson2021importance} trains multiple shadow models by following the same training procedure of the target model $f$ on the auxiliary dataset. It then calibrates the score of $z$ by subtracting the average score of $z$ among shadow models. 
Another representative method LiRA~\citep{carlini2022membership} is to estimate density functions for the two distributions of scores of z when the model is trained \emph{with} / \emph{without} the input data $z$, where the scoring function is pre-defined. After calculating the score of $z$ and $f$, it computes the ratio of density values of the two distributions and thresholds this ratio.

While MIAs that use shadow models tend to have higher advantage than score-based ones, they can be impractical because of two reasons. First, their performance is very sensitive to the knowledge of the adversary, including the distribution of auxiliary datasets and the details of the learning procedure~\citep{carlini2022membership, duan2024membership}. 
The second and more important aspect is computational cost. For modern models, training even a single model can take multiple GPUs and many days -- Llama2-70B takes 1720320 GPU hours for example~\citep{touvron2023llama} -- which makes it infeasible to train multiple (or even one) shadow model. With this in mind, we focus our attention to pure score-based MIAs in this paper.

\begin{table}[t]
 \centering
 \caption{MIA scores in the literature.}
\resizebox{\linewidth}{!}{\
\begin{tabular}{cc}
\toprule
Name & Definition\\
\midrule
Maximum-Softmax-Probability (MSP)~\citep{yeom2018privacy} & $-\max_{c\in [C]}f(x)_c$\footnote{We take the negative sign here because the smaller score implies the data more likely belonging to the training data as defined in Equation~\ref{eq:thr_att}.}\\
Entropy (ENT)~\citep{shokri2017membership} & $\sum_{c\in [C]}-f(z)_c\log(f(x)_c)$\\
Cross-Entropy Loss (CE)~\citep{yeom2018privacy} & $\sum_{c\in [C]}-y_c\log(f(x)_c)$\\
Modified Entropy (ME)~\citep{song2021systematic} & $-\sum_{c\in [C]}\left((1-f(x)_c)\log(f(x)_c)y_c + f(x)_c\log(1-f(x)_c)(1-y_c)\right)$\\
\bottomrule
\end{tabular}
}	
\label{tab:mia_score}
\end{table}

\section{A Better Empirical Privacy Metric}
\label{sec:discrepancy}



Recall from Section~\ref{sec:preliminary} that the advantage of an MIA on a model $f$ is an empirical privacy metric that measures how much $f$ ``leaks'' its training data, and that different scoring functions yield different empirical privacy metric values depending on the model and dataset. This raises a natural question: can we find an empirical privacy metric that encompasses all these scores-based advantages?

Before defining the metric, let us first discuss three properties that we expect from it.  First, it should be an upper bound on the advantage of a family of score-based MIAs and this indicates that it is a reasonably {\em{strict}} privacy metric to ensure a model is privacy-preserving -- the advantage of any MIA in this family would not exceed this upper bound.   
Second, the metric should be able to distinguish between different models so that we can use it to compare models by their privacy leakage; one counter-example is that the family of score-based MIAs are so expressive that training and test data can always be perfectly separated. Finally, it should be efficiently computable, or at the very least, approximately computable with relative ease.

\subsection{Better Privacy Metric through Discrepancy Distance}
\paragraph{Connecting MIA to discrepancy distance.} 
\citet{yeom2018privacy} showed that the advantage of a loss-function-based MIA is equal to the generalization gap between the training and test loss. But what happens when we look at, not a single score, but a family of scores?


Suppose $(x, y)$ is a labeled data and $f(x) \in\Delta_{C-1}$ is the softmax vector of probabilities output by a $C$-class classification model. For any MIA $m$ that is a post-hoc function of $(f(x), y)$, we can define the discriminative set of $m$ as
$$Q_m:=\{(f(x), y)|(x, y)\in\mathrm{supp}(\calD) \text{ and } m((x, y),  f)=1\},$$
where $\mathrm{supp}(\calD)$ is the support of data distribution $\calD$. This is essentially the set where the scoring function underlying the MIA predicts that $(x, y)$ lies in the training set. We can now rewrite the advantage of a MIA $m$ in terms of its discriminative set as follows: 
\begin{align*}
	\label{eq:discrepancy}
    {\rm Adv}(m; f, S, \mathcal{D}) = \bbP_{(x, y)\sim S}((f(x), y)\in Q_m) - \bbP_{(x, y)\sim \calD}((f(x), y)\in Q_m).
\end{align*}
Now, suppose $\calQ$ is a family of discriminative sets and $Q_m\in\calQ$; then, the advantage of $m$ is naturally bounded by the discrepancy distance~\citep{mansour2009domain, doerr2014calculation} between the training set $S$ and the test distribution $\calD$ with respect to $\calQ$, which is defined as 
$$
D_{\calQ}(S, \calD) := \sup_{Q\in \calQ} D(S, \calD|Q), 
$$
where $D(S, \calD|Q)=\left|\bbP_{(x, y)\sim S}((f(x), y)\in Q) - \bbP_{(x, y)\sim \calD}((f(x), y)\in Q)\right|.$ Formally, we state the relation between the advantage and discrepancy in the following proposition.
\begin{proposition}
\label{prop:discrepancy}
For any MIA $m$, if $Q_m\in \calQ$, ${\rm Adv}(m; f, S, \mathcal{D})\leq D_{\calQ}(S, \calD)$.	
\end{proposition}

\paragraph{Choosing a discriminative set family $\calQ$.} 


What are some of the popular scoring functions for MIA?
Two of the most popular ones are MSP and ENT, both of which are based on scores that are convex or concave functions of $(f(x), y)$ and hence have convex discriminative sets. This suggests convex discriminative sets $\calQ_{cvx}=\{Q|Q\subseteq \bbR^{2C}, Q\text{ is a convex set}\}$ as a potential candidate. Are there any more such scoring functions? It turns out that two other popular ones -- CE and ME -- while not based on convex functions -- can be shown to have advantages that are equal to that of a convex or concave scoring function; hence, the discrepancy over convex discriminative sets is an upper bound on their advantage as well. This is encapsulated by the following theorem.


\begin{theorem}
\label{thm:cvx_dis_up}
	For an arbitrary threshold $\tau\in\bbR$ and any of $m\in\{m_{msp, \tau}, m_{ent, \tau}, m_{ce, \tau}, m_{me, \tau}\}$, 
	${\rm Adv}(m; f, S, \mathcal{D})\leq D_{\calQ_{cvx}}(S, \calD).$
\end{theorem}
\begin{proof}[Proof sketch of Theorem~\ref{thm:cvx_dis_up}]
	The proof for $m_{msp, \tau}$ and $m_{ent, \tau}$ is straightforward because the MSP score and the ENT score are convex/concave in $(f(x), y)$, hence thresholding the convex/concave function gives a convex/concave discriminative set.
	The proof for $m_{ce, \tau}$ and $m_{me, \tau}$ takes more effort because the cross-entropy score and modified-entropy score are indeed not convex/concave in $(f(x), y)$.
	For each of $m_{ce, \tau}$ and $m_{me, \tau}$, we can construct another function $g(f, z)$ such that 1. it agrees with the score (CE or ME) for any $(f, z)$ and 2. it is convex or concave in $(f(x), y)$, $\forall$ $f(x), y\in (0, 1]^C\times [0, 1]^C$. 
    Hence, the discriminative set of $g(f,z)$ is the same as the discriminative set of the score (CE or ME) in domain $ (0,1]^C \times \{0,1 \}^C$ and is convex in domain $(0,1]^C \times [0,1]^C$.
	This is possible because by definition the label $y$ of $z$ should be a one-hot vector and the convexity is discussed for a larger domain.
	With this construction, ${\rm Adv}(m_{\mathrm{score}, \tau}; f, S, \mathcal{D}) = {\rm Adv}(m_{g, \tau}; f, S, \mathcal{D}) = \mathbb{P}_{z\sim S}((f(x),y) \in Q_{m_{g, \tau}}) - \mathbb{P}_{z\sim \mathcal{D}}((f(x),y) \in Q_{m_{g, \tau}}) \leq D_{\calQ_{cvx}}(S, \calD)$ for $\mathrm{score}\in\{ce, me\}$.
\end{proof}
It is worthwhile to highlight from the theorem that ${\rm Adv}(\calM_{cvx}; f, S, \mathcal{D})$ can be proved to be an upper bound for the advantage of an MIA $m$ even when $Q_m$ is neither convex nor concave --- $Q_{m_{CE}}$ and $Q_{m_{ME}}$ are the examples as proved in Theorem~\ref{thm:cvx_dis_up}.

We propose $D_{\calQ_{cvx}}(S, \calD)$ as a new privacy metric and now we turn to the three criteria discussed at the beginning of Section~\ref{sec:discrepancy}.
\textbf{First}, from Proposition~\ref{prop:discrepancy}, $D_{\calQ_{cvx}}(S, \calD)$ is an upper bound for all MIA, whose discriminative sets are convex or concave; Theorem~\ref{thm:cvx_dis_up} further shows MIAs with four popular existing scores, where two of them are even not convex or concave functions, are upper bounded by $D_{\calQ_{cvx}}(S, \calD)$.
\textbf{Second}, observe that as 
$\calQ_{cvx}$ is not arbitrarily expressive -- $Q\in\calQ_{cvx}$ is constrained to be a convex set in $\bbR^{2C}$, $D_{\calQ_{cvx}}(S, \calD)$ should not be loose enough to be a trivial upper bound. In particular, our numerical experiments will illustrate 
that this upper bound $D_{\calQ_{cvx}}(S, \calD)$ is non-vacuous in many common cases and is capable of distinguishing between different models and different datasets.
\textbf{Third}, we are going to show that an approximation of $D_{\calQ_{cvx}}(S, \calD)$ can be computed efficiently; this is the topic of the next subsection.

\paragraph{Comparison with the existing MIAs.} We now compare our new metric to two categories of the existing MIAs introduced in Section~\ref{sec:preliminary}. As proved in Theorem~\ref{thm:cvx_dis_up}, the new metric is stronger than the popular score-based MIAs. On the other hand, because its computation doesn't involve training multiple ``shadow models", it is more computational-feasible than the MIAs leveraging ``shadow models" and can be applied to large models.

\subsection{Approximation of the discrepancy} 
\label{sec:approximation of the discrepancy}

In general, it is challenging to represent arbitrary convex sets, and hence a natural strategy is to approximate the set $\calQ_{cvx}$ by a set $\calQ_{cvx, K}$ of all convex polytopes with $K$ facets for large $K$. It turns out that there exists a $K$ such that this approximation is exact when considering closed sets, suggesting this is a viable solution strategy. 

\begin{proposition}
\label{thm:cvx_up2}
Suppose $\calQ_{cvx}'$ and $\calQ_{cvx, K}'$ are the sets of all closed convex sets and closed convex polytopes respectively. We have $D_{\calQ_{cvx}'}(S, \calD) = D_{\calQ_{cvx, K}'}(S, \calD)$ for $K=\binom{|S|}{2C}$.
\end{proposition}

Unfortunately, exactly calculating the discrepancy distance $D_{\calQ_{cvx}}(S, \calD)$ over even polytopes with $2$ facets is NP-hard~\citep{gottlieb2018learning}! 
To overcome this hardness, 
we first observe that each convex polytope with $K$ facets can be parameterized as $Q_{w_i, b_i, i\in[K]}:=\{a|a\in\bbR^{2C}, w_i^\top a+b_i\leq 0, \text{ where } w_i\in\bbR^{2c}, b_i\in\bbR, i\in[K]\}$.
After this parameterization, we can use a standard technique in machine learning -- instead of optimizing over the $0/1$ loss in discrepancy, we instead use a smoother surrogate loss and optimize over it. 
Specifically, we choose logistic regression $\ell_{lg}$.
By notating $a_{z, f}=(f(x), y)$, the objective function can be written as
\begin{equation}
\label{eq:cpm}
\max_{w_i\in\bbR^{2C}, b_i\in\bbR, i\in[K], s=\pm1}
\frac{1}{|S|}\sum_{z\sim S}\ell_{lg}\left(\max_{i\in[K]}w_i\top a_{z, f}+b_i, s\right) + \bbE_{z\sim \calD}\left[\ell_{lg}\left(\max_{i\in[K]}w_i\top a_{z, f}+b_i, -s\right)\right].
\end{equation}
Now the objective function is both parametric and continuous.
Although it is still non-convex, we can use gradient descent to find the approximate solution.

To distinguish the exact solution $(w_i^*, b_i^*)$ representing the best polytope and the approximate solution $(\hat{w}_i, \hat{b}_i)$ solved by Equation~\ref{eq:cpm}, we name the optimal value $D(S, \calD|Q_{w_i^*, b_i^*, i\in[K]})$ as the CPB (Convex Polytope Bound) and by following \citet{kantchelian2014large} name $D(S, \calD|Q_{\hat{w}_i, \hat{b}_i, i\in[K]})$ as the CPM (Convex Polytope Machine). Moreover, we notice that CPM is achievable by some score-based MIA: $D(S, \calD|Q_{\hat{w}_i, \hat{b}_i, i\in[K]})$ is equivalent to the advantage of a score-based MIA with scoring function $\max_{i\in[K]}\hat{w_i}\top a_{z, f}+\hat{b_i}$ or $-\max_{i\in[K]}\hat{w_i}\top a_{z, f}+\hat{b_i}$.

\section{Experiment}\label{Sec:Experiment}
In this section, we investigate the empirical performance of CPM on both various models and several datasets, where we follow the setups in the MIA literature~\citep{chen2022relaxloss}, and pre-trained ImageNet classification models in the wild. In particular, we are interested in the following questions:
\begin{enumerate}[leftmargin=*,nosep]
    \item How does CPM perform compared to other scoring functions on models trained with different learning algorithms?
    \item How good is the approximation quality of CPM with $K$ facets as a function of $K$? 
    \item How does CPM perform on models in the wild?
\end{enumerate}

\subsection{Experimental Setup}

\begin{figure}[t!]
    \centering
    \subfigure[CIFAR-10]{\includegraphics[width=0.45\textwidth]{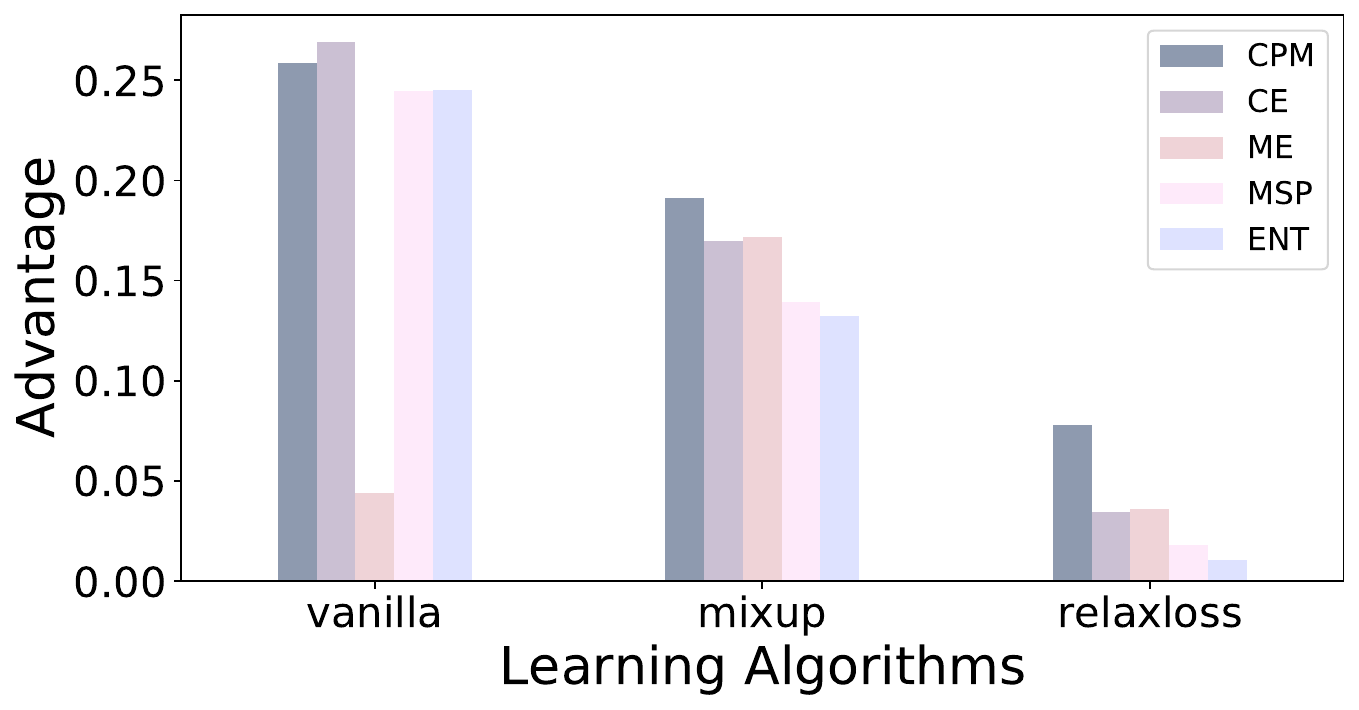}
    \label{fig:CIFAR_10}}
    \hfill
    \subfigure[CIFAR-100]{\includegraphics[width=0.45\textwidth]{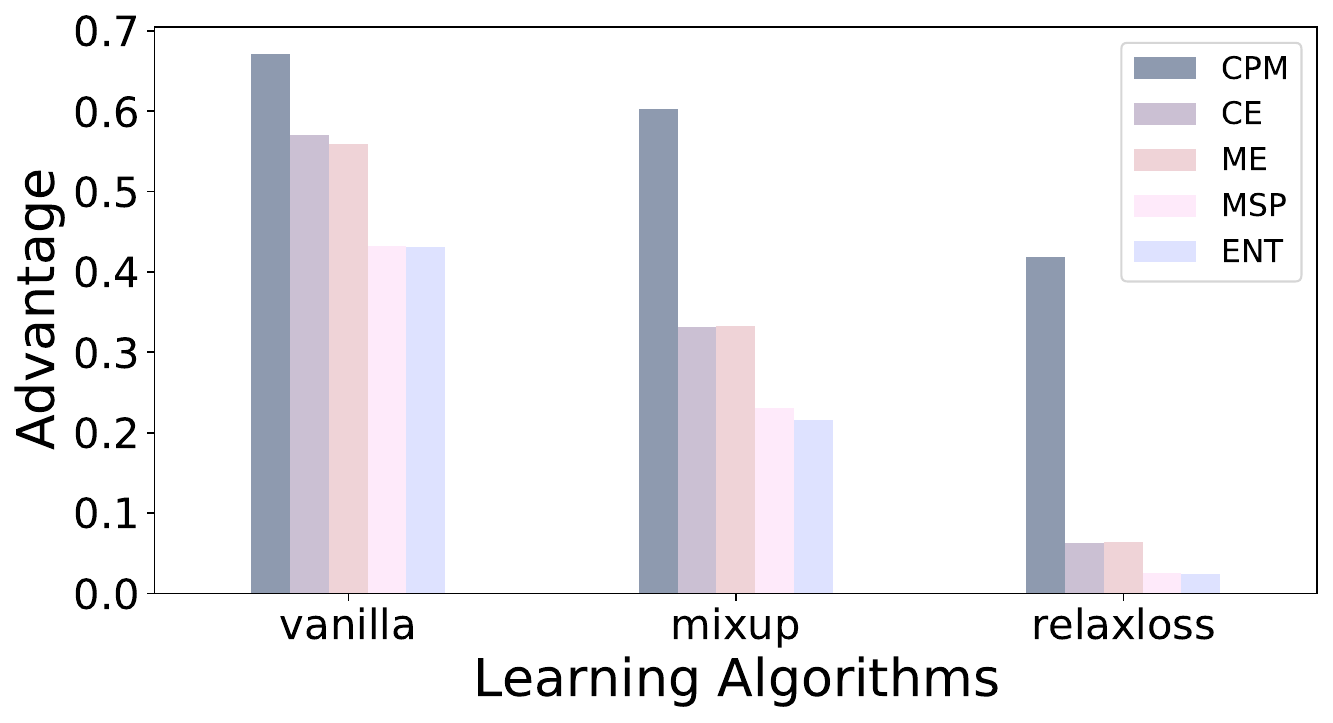}
    \label{fig:CIFAR_100}}\\
    \vspace{-2ex}
    \subfigure[Texas]{\includegraphics[width=0.45\textwidth]{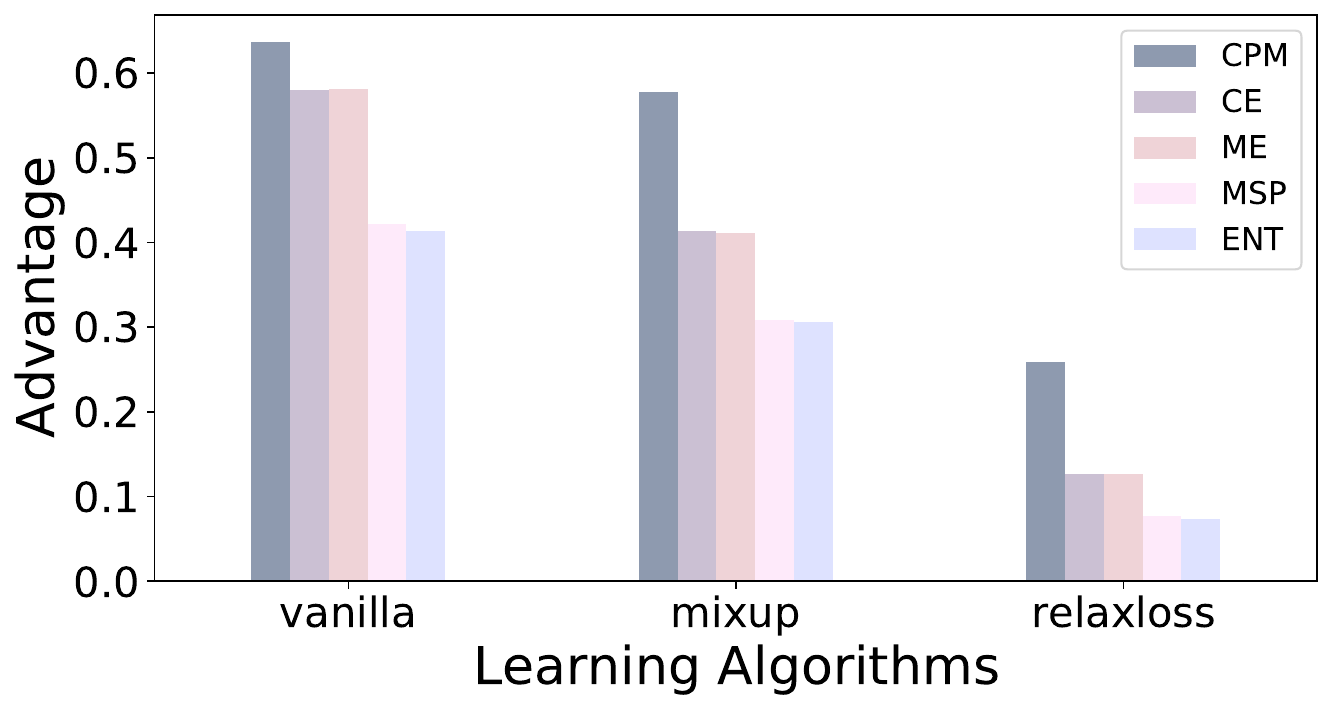}
    \label{fig:Texas}}
    \hfill
    \subfigure[Purchase]{\includegraphics[width=0.45\textwidth]{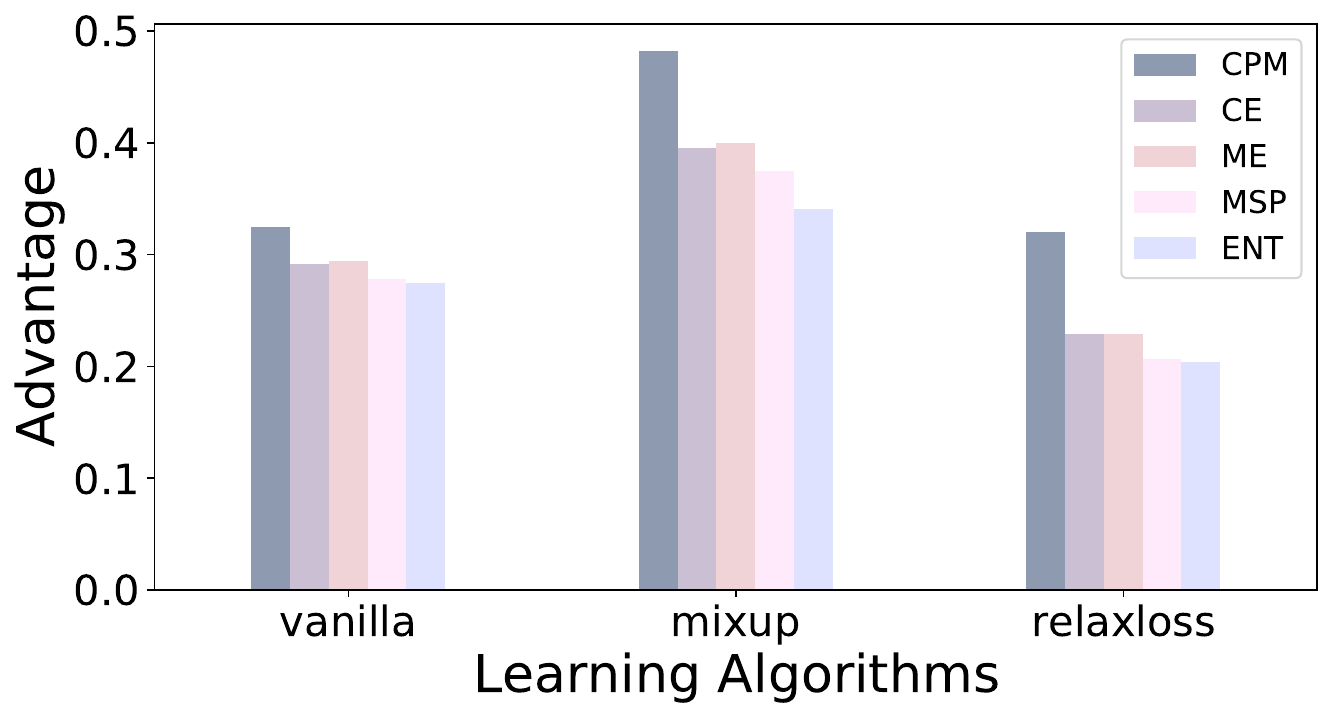}
    \label{fig:Purchase}}
    \caption{CPM and the advantage of baselines on models trained on CIFAR-10 \ref{fig:CIFAR_10}, CIFAR-100 \ref{fig:CIFAR_100}, Texas \ref{fig:Texas} and Purchase \ref{fig:Purchase}. As shown in the figures, CPM is an upper bound to the advantage of the baseline scores for most models.}
    \vspace{-3ex}
    \label{fig:full}
\end{figure}

\textbf{Datasets and Models.} 
We first consider two image classification datasets: CIFAR-10~\citep{cifar}, CIFAR-100~\citep{cifar} and two tabular datasets~\citep{shokri2017membership}: Texas100, Purchase100. The CIFAR-10 has 10 classes and the other three datasets have 100 classes.
The following settings for these four datasets and the to-be-test models mainly follow the MIA literature~\cite{carlini2022membership, chen2022relaxloss} and we call these models by \textit{models trained from scratch}.
We use ResNet-20~\citep{resnet} as the model architecture for CIFAR-10 and CIFAR-100, and MLP ~\citep{haykin1994neural} as the model architecture for Texas100 and Purchase100. Models are trained with 10000 balanced training samples. 
We use three different training methods. The first, \emph{vanilla}, is trained by minimizing the standard cross-entropy loss. \emph{Mixup}~\citep{zhang2018mixup} is trained by minimizing the cross-entropy loss after linear interpolation of the training data and labels -- a method that is known to promote generalization. Finally, our third method \emph{RelaxLoss}~\citep{chen2022relaxloss} trains models by minimizing the the cross-entropy loss in a dynamic way, which is a state-of-the-art empirical defense against MIAs. These three methods lead to models with accuracy varying between $78-84\%$ for CIFAR-10 and $39-52\%$ for CIFAR-100, $52-58\%$ for Texas-100 and $78-89\%$ for Purchase-100.

We further test our methods on ImageNet~\citep{deng2009imagenet} dataset and pre-trained Imagenet models downloaded from the publicly available Pytorch Torchvision library ~\citep{paszke2017PyTorch}; surprisingly, there are no previous published MIA results on these models.
We pick ResNet-50, ResNet-101 and ResNet-152 models from Pytorch~\cite{paszke2019pytorch} pre-trained on ImageNet with both versions 1 and version 2\footnote{The models are from \url{https://pytorch.org/vision/stable/models.html}.}.The version 1 models are trained by minimizing the usual cross-entropy loss, while the version 2 models, which are more accurate, use an advanced training recipe that includes many generalization-promoting techniques such as Label Smoothing~\citep{muller2019does}, Mixup, Cutmix~\citep{yun2019cutmix}, Random Erasing~\citep{zhong2020random} and so on. 

\textbf{CPM setup.}  we get our new metric CPM by optimizing the objective in Equation~\ref{eq:cpm} with $K = 1000$. All optimizations are conducted with GPU on NVIDIA GeForce RTX 3080, and the longest time for one optimization is 20 minutes. See optimization details in Appendix~\ref{sec:app_exp}.


\textbf{Baselines.} 
We compare CPM with four popular baseline scores used in the literature: maximum-softmax-probability (MSP), entropy (ENT), cross-entropy loss (CE) and modified entropy (ME); see Table~\ref{tab:mia_score} for their definitions.

\textbf{Evaluation Method.} To evaluate the advantage (from equation \ref{eq:adv}) of MIAs between training samples and testing distribution, for each target model, the CPM is computed by optimizing Equation~\ref{eq:cpm} on the entire training set and half of the test set. 
We then report the actual advantage of the CPM calculated over the training set and the rest of the testing samples.
Similarly, for the other scores, we choose the optimal threshold based on the training set and half of the testing samples, and evaluate the actual advantage similarly.

\subsection{Observations}

\begin{figure}[t!]
    \centering
    \subfigure[Vanilla]{\includegraphics[width=0.3\textwidth]{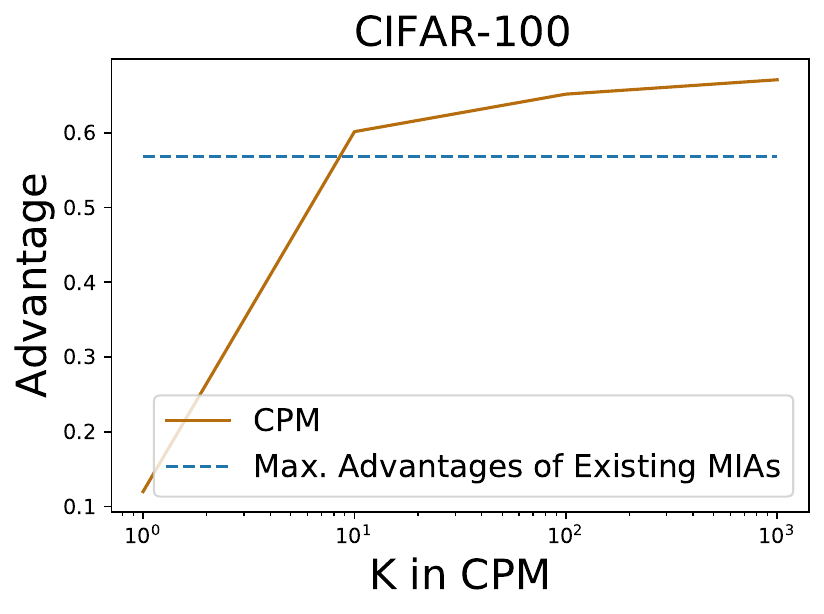}\label{fig:vanilla_abla_cifar100}}
    \hfill
    \subfigure[MixUp]{\includegraphics[width=0.3\textwidth]{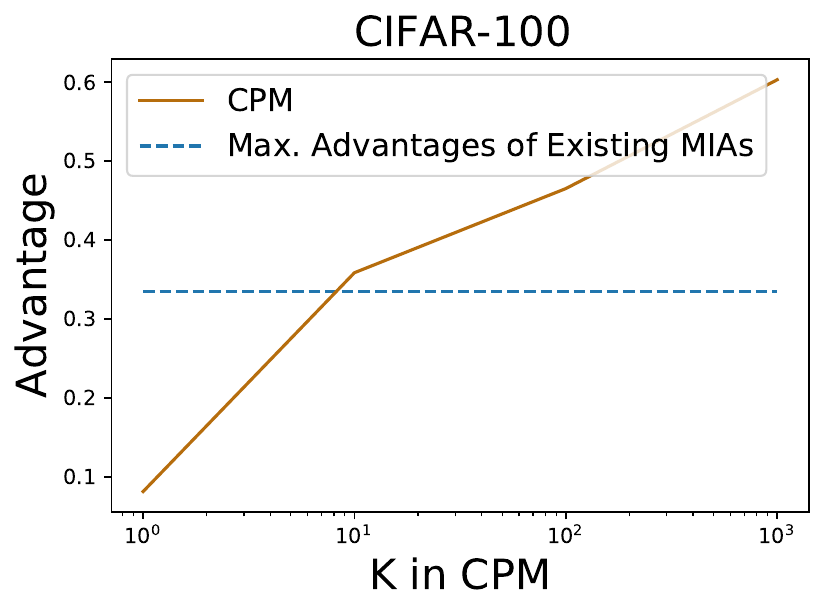}\label{fig:mixup_abla_cifar100}}
    \hfill
    \subfigure[Relaxloss]{\includegraphics[width=0.3\textwidth]{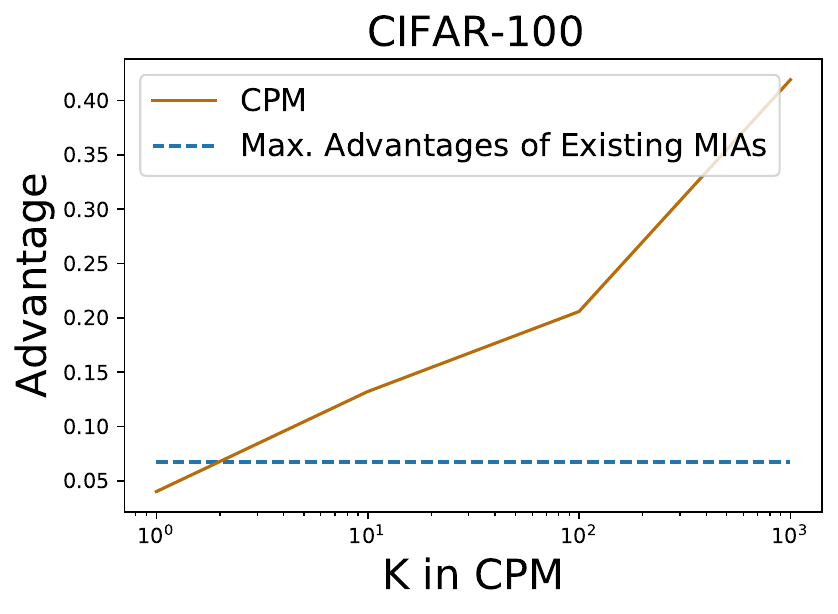}\label{fig:relaxloss_abla_cifar100}}
    \hfill
    \captionof{figure}{CPM with different numbers of facets $K$. The figures show CPM achieves a higher advantage than the existing MIAs as an uppper bound with a moderate value of $K$. }
    \vspace{-2ex}
   \label{fig:ablation}
\end{figure}

\begin{wrapfigure}{O}{0.45\textwidth}
\centering
\vspace{-4ex}
\includegraphics[width=0.45\textwidth]{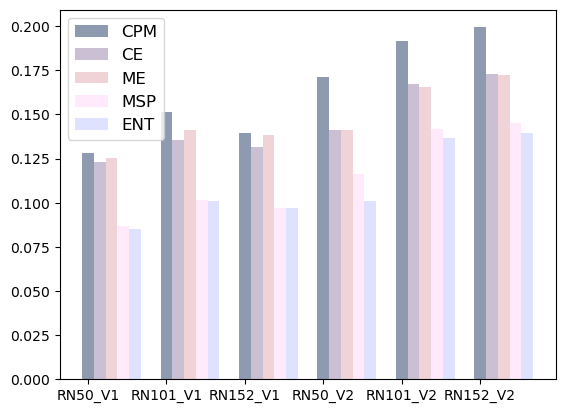}
\caption{\label{fig:imagenet}CPM and the advantage of baseline scores on PyTorch models. It shows that CPM is very close to the advantage of the baseline scores for the V1 models, but the gap is significantly larger for the V2 models.}
\vspace{-2ex}
\end{wrapfigure}

\textbf{Models trained from scratch.} The results for CIFAR-10, CIFAR-100, Texas-100 and Purchase-100 are presented in Figures~\ref{fig:CIFAR_10}, ~\ref{fig:CIFAR_100}, ~\ref{fig:Texas} and ~\ref{fig:Purchase} respectively. We see that for most models, CPM is an upper bound on the advantage of the four baseline scores. This corroborates Theorem \ref{thm:cvx_dis_up}, and shows that the approximate solution CPM for the upper bound serves as a stronger privacy metric. Note that only for the vanilla model in CIFAR-10, Cross-Entropy loss achieves slightly higher advantage than CPM. This might be because CPM {\em{approximates}} the theoretical upper bound CPB, which suggests that the difference might be due to approximation error.

More interestingly, we observe that CPM is very close to the advantage of the Cross-Entropy loss in vanilla models, while there is a sizeable gap between the advantages for Mixup and Relaxloss models.
Recall that CPM is achievable by some score-based MIA as discussed in Section~\ref{sec:approximation of the discrepancy}.
This suggests the design of scores such as Cross-Entropy and MSP, which are used because of their high empirical performance on cross-entropy trained models, might be overfitting to these kinds of models. Different scores may be needed for more effective membership inference in more sophisticated models. 


\textbf{Effect of $K$ for approximation quality of CPM.} Recall that the parameter $K$ in the CPM computation measures approximation quality. We plot CPM on CIFAR-100 versus the number of facets $K$ in Figure~\ref{fig:ablation}. We see that this is a monotone increasing function where larger $K$ has higher advantage, suggesting that CPM is approaching the true upper bound CPB well when $K$ is getting larger.  Furthermore, we see that the CPM outperforms the advantage of other scores even at $K = 10$, which suggests that in these cases, even a reasonably large value of $K$ can achieve larger advantage. We make the similar evaluation for $K$ on other three datasets in Appendix~\ref{sec:app_exp}.

\textbf{Models in the Wild.} The results for PyTorch models are shown in Figure \ref{fig:imagenet}. We observe that CPM is very close to the advantage of the baseline scores for the Version 1 models which are trained by minimizing the cross-entropy loss, whereas the gap is significantly larger for the version 2 models, which are trained with a more advanced recipe.
This shows that our observation from Figure \ref{fig:full} that the existing baseline scores may not be very effective for membership inference in more sophisticated models also holds for pre-trained models in the wild.

\section{Improving MI Attacks for More Sophisticated Models}

Our results from Section \ref{Sec:Experiment} show that while the current loss-based membership inference attacks are effective for cross-entropy trained models, they are considerably less so for the more sophisticated models of today. Thus, a natural question to ask is whether there are other better score-based membership inference attacks for the models trained through different procedures? 

This suggests us to design training-procedure aware scores. Although this requires the additional knowledge for the training procedure, this assumption can be realistic -- the public foundation models, for example Pytorch pre-trained ImageNet classification models and public large language models (GPT~\citep{radford2019language}, Llama~\citep{touvron2023llama}), are always accompanied with the certain level of training details -- and this assumption is much easier than other common assumptions in literatures that require either the full model weights or the additional auxiliary dataset.

Also, notice that the training-procedure aware score is different from the loss aware score -- for example in the previous experiments, vanilla, MixUp and RelaxLoss models are all trained based on cross-entropy loss, but their training-procedures are different. 
We are going to design training-procedure aware scores next for the MixUp and RelaxLoss models.

For Mixup models, the training loss is the cross-entropy loss of a ``Mixup'' -- or, linear combination -- of two training points. This suggests the following Mixup score. 
Suppose we have a small auxiliary subset $S_{\rm aux}$ of the training data; \footnote{In our evaluation $|S_{\rm aux}|$ is smaller than $200$.} then we define the Mixup score $ m_{\rm MixUp}(z, f)$, as the CE loss between a linear combination of $z$ and $z_{\rm aux} \in S_{\rm aux}$:
$$
\frac{1}{R|S_{\rm aux}|}\sum_{r\in[R], z_{\rm aux}\in S_{\rm aux}}\ell_{\rm ce}(f(x_{\rm mix}), y_{\rm mix}),
$$
where $x_{\rm mix}=\lambda_{r}\cdot x + (1-\lambda_r)\cdot x_{\rm aux}$, $y_{\rm mix}=\lambda_{r}\cdot y + (1-\lambda_r)\cdot y_{\rm aux}$, and $\lambda_r$ ($r\in[R]$) are i.i.d. sampled from a uniform distribution $\mathcal{U}_{[0.5, 1]}$.

For RelaxLoss models, there is no static loss function cross optimization iterations and the training procedure is more complicated. If the loss on an example is more than $\alpha$, then~\citet{chen2022relaxloss} does standard SGD. Otherwise, if the classifier predicts correctly on the example, then we do gradient ascent; if not, then they do gradient descent on a modified loss function that incorporates a smoothed label. We mimic this dynamic training procedure into the following RelaxLoss score $m_{\rm Relaxloss}(z, f)$: 
\begin{align}
\label{eq:relaxloss}
	|\ell_{\rm ce}(f(x), y)-\alpha| + \left(1.5 - \ell_{0/1}(f, z)\right) \cdot \ell_{\rm ce}(f(x), y) + \left(0.5 + \ell_{0/1}(f, z)\right) \cdot \ell_{\rm ce-s}(f(x), y),
\end{align}
where $\alpha$ is a hyperparameter, $ \ell_{0/1}$ is the classification error, and $\ell_{\rm ce-s}(f(x), y):=\sum_{c=1}^C \log(f(x)_c) y^{\rm soft}_c)$ is the cross-entropy loss after label smoothing as defined in \citet{chen2022relaxloss}. 
The soft label $y^{\rm soft}$ is defined as $y^{\rm soft}_{c^*}=\min\{f(x)_{c^*}, \mu\}$ where $y_{c^*}=1$ and $y^{\rm soft}_c=\frac{1-\min\{f(x)_{c^*}, \mu\}}{C-1}$ for any $c\neq c^*$ and $\mu$ is another hyperparameter to soften the original one-hot label.

\subsection{Experiments with Mixup and RelaxLoss Scores}
In this section, we implement the Mixup score and RelaxLoss score attack on pre-trained models from the previous section. We are interested in the following question: Do the Mixup score and RelaxLoss score outperform existing MIA scores for Mixup and RelaxLoss models respectively?

\textbf{Experimental Setups.} For the Mixup experiments, we implement a number of values of $|S_{aux}|$ and $R$. For two image datasets, we randomly sample our auxiliary dataset $S_{aux}$ from the test set with $|S_{aux}|$ equal to 30 and $R = 10$. For tabular datasets, we randomly sample $S_{aux}$ from the train set with $|S_{aux}| = 100, R = 100$ for Purchase-100, and $|S_{aux}| = 150, R = 150$ for Texas-100. For the RelaxLoss experiments, we set the hyperparameters $(\alpha, \mu)$ in Equation~\ref{eq:relaxloss} as what they were used in the RelaxLoss training procedure~\citep{chen2022relaxloss}. The two values of $(\alpha, \mu)$ are $(1, 1), (3, 1), (2.5, 0.1), (0.8, 0.3)$ for CIFAR-10, CIFAR-100, Texas-100 and Purchase-100 respectively.

\textbf{Observations.} The results for CIFAR-10, CIFAR-100, Texas-100 and Purchase-100 are presented in Table \ref{tab:CIFAR-10}, \ref{tab:CIFAR-100}, \ref{tab:Texas}, \ref{tab:Purchase-100} respectively. We see that as we expected, the Mixup score and the RelaxLoss score have the highest advantage for the Mixup and the RelaxLoss model respectively. In other words, the advantage is the highest when the training procedure and the MIA scores are aligned. This suggests that for the modern models, that are trained with more sophisticated procedures, we may be able to design better MIA scores that are training-procedure aware.

\begin{table}[t!]
    \centering
    \small
    \subtable[CIFAR-10]{
        \begin{tabular}{|c|c|c|c|} \hline
        Model $\backslash$ Score  & CE & Mixup & RelaxLoss   \\ \hline
        Vanilla  & \textbf{28.45} & 15.11 & 12.53 \\
        Mixup &18.50 & \textbf{18.92} & 13.80\\
        RelaxLoss & 4.9 & 4.75 & \textbf{5.24}
        \\ \hline
    \end{tabular}
    \label{tab:CIFAR-10}
    }
    \quad
    \subtable[CIFAR-100]{
        \begin{tabular}{|c|c|c|c|} \hline
        Model $\backslash$ Score  & CE & Mixup & RelaxLoss   \\ \hline
        Vanilla  & \textbf{56.80} & 50.60 & 48.65 \\
        Mixup &33.36 & \textbf{33.95} & 31.89\\
        RelaxLoss & 6.62 & 6.73 & \textbf{8.64} \\\hline
    \end{tabular}
    \label{tab:CIFAR-100}
    }
    \quad
    \subtable[Texas]{
        \begin{tabular}{|c|c|c|c|} \hline
        Model $\backslash$ Score  & CE & Mixup & RelaxLoss   \\ \hline
        Vanilla  & \textbf{58.53} & 58.52 & 46.91\\
        Mixup & 41.16 & \textbf{41.25} & 35.34\\
        RelaxLoss & 13.10 & 13.09 & \textbf{13.45}
        \\ \hline
    \end{tabular}
    \label{tab:Texas}
    }
\quad
\subtable[Purchase]{
\begin{tabular}{|c|c|c|c|} \hline
        Model $\backslash$ Score  & CE & Mixup & RelaxLoss   \\ \hline
        Vanilla  & \textbf{29.17} & 22.77 & 11.07 \\
        Mixup &39.11 & \textbf{39.12} & 32.77\\
        RelaxLoss & 22.82 & 22.77 & \textbf{22.86}
        \\ \hline
    \end{tabular}
    \label{tab:Purchase-100}
}
    \caption{The advantage of MIA with cross-entropy score, Mixup score, and RelaxLoss score (columns) for models (rows) trained on CIFAR-10 \ref{tab:CIFAR-10}, CIFAR-100 \ref{tab:CIFAR-100}, Texas \ref{tab:Texas} and Purchase \ref{tab:Purchase-100}. The tables show that the advantage is the highest when the training loss and the MIA scores are aligned. }
    \vspace{-4ex}
    \label{tab:fulltable}
\end{table}
\section{Related Work}
\textbf{Membership inference attack in large-scale foundation models.} 
While this paper mainly discusses the membership inference attack for classification models, many literature study the membership inference attack for more recent large-scale foundation models.
Several work study the MIA for the unsupervised models such as GAN~\citep{chen2020gan}, contrastive learning~\citep{liu2021encodermi, he2021quantifying} and diffusion models~\citep{matsumoto2023membership, duan2023diffusion}.
MIA has been explored for multi-modal models such as CLIP~\citep{ko2023practical, hintersdorf2022does}, text-to-image generation models~\citep{wu2022membership}, and image captioning models \citep{hu2022m}.
\citet{mireshghallah2022quantifying, mattern2023membership, fu2023practical} explore MIA for the latest large language models.

\textbf{Attribute inference attack.}
Another popular empirical privacy metric is attribute inference, where the input is the partial knowledge of a training data point and the model, and the output is an estimation of the unknown sensitive attributes.
Attribute inference has been studied in various data domains.
\citet{fredrikson2014privacy, yeom2018privacy, mehnaz2022your} study the attribute inference on tabular dataset,
\citet{zhang2020secret, aivodji2019gamin, meehan2024ssl} focus on the unsupervised image models, and \citet{jia2017attriinfer, gong2018attribute} explore the attribute inference in social networks.

\textbf{Differential privacy.} 
Differential privacy~\citep{dwork2006calibrating,dwork2014algorithmic} is a theoretical privacy definition for any learning algorithm.
Different from empirical privacy metrics which are measured by any designed attacks, the DP parameters for the learning algorithm, which indicates the privacy leakage, need to be proved.
\citet{yeom2018privacy, humphries2023investigating, wu2023does} have discussed the relationship between differential privacy and empirical privacy metrics such as membership inference attack or attribute inference attack.
On the other hand, its good theoretical property motivates the design of privacy-preserving algorithms and it has been widely deployed in many tasks~\citep{abadi2016deep, iyengar2019towards, de2022unlocking, yu2021differentially}.

\section{Conclusion}\label{sec:Conclusion}
In this work, we proposed a new empirical privacy metric based on discrepancy distance, i.e. the discrepancy distance between the training and test data with respect to a class of sets. We show that this is a stronger privacy metric than four current score-based MIAs and gives an informative upper bound. 
In addition, we introduce an approximation CPM to the new metric, that is computationally efficient. 
In our experiment section, we observe that existing MIAs are upper-bounded by the CPM.
More interestingly, we find that the design of existing score-based MIA may overfit standard models and is suboptimal for the model trained by sophisticated recipes. 
Motivated by this observation, we propose two new score-based MIAs for MixUp and RelaxLoss, which achieve higher advantages. 

\textbf{Limitation and future work.} First, this paper considers the discrepancy distance with respect to the convex family of sets in the probability space. One extension can be either considering a more general family than the convex family or exploring the discriminative set in the logit space or the feature space.
Another future direction motivated by this work is to design better score-based MIA for models trained by more sophisticated learning techniques in various data domains. 

\paragraph{Acknowledgement}
This work is supported by grants from National Science Foundation NSF (CIF-2402817, CNS-1804829), SaTC-2241100, 
CCF-2217058, ARO-MURI (W911NF2110317), and ONR under N00014-24-1-2304.  


\newpage
\bibliographystyle{abbrvnat}
\bibliography{main}

\begin{thebibliography}{61}
\providecommand{\natexlab}[1]{#1}
\providecommand{\url}[1]{\texttt{#1}}
\expandafter\ifx\csname urlstyle\endcsname\relax
  \providecommand{\doi}[1]{doi: #1}\else
  \providecommand{\doi}{doi: \begingroup \urlstyle{rm}\Url}\fi

\bibitem[Abadi et~al.(2016)Abadi, Chu, Goodfellow, McMahan, Mironov, Talwar, and Zhang]{abadi2016deep}
M.~Abadi, A.~Chu, I.~Goodfellow, H.~B. McMahan, I.~Mironov, K.~Talwar, and L.~Zhang.
\newblock Deep learning with differential privacy.
\newblock In \emph{Proceedings of the 2016 ACM SIGSAC conference on computer and communications security}, pages 308--318, 2016.

\bibitem[A{\"\i}vodji et~al.(2019)A{\"\i}vodji, Gambs, and Ther]{aivodji2019gamin}
U.~A{\"\i}vodji, S.~Gambs, and T.~Ther.
\newblock Gamin: An adversarial approach to black-box model inversion.
\newblock \emph{arXiv preprint arXiv:1909.11835}, 2019.

\bibitem[Carlini et~al.(2019)Carlini, Liu, Erlingsson, Kos, and Song]{carlini2019secret}
N.~Carlini, C.~Liu, {\'U}.~Erlingsson, J.~Kos, and D.~Song.
\newblock The secret sharer: Evaluating and testing unintended memorization in neural networks.
\newblock In \emph{28th USENIX security symposium (USENIX security 19)}, pages 267--284, 2019.

\bibitem[Carlini et~al.(2022)Carlini, Chien, Nasr, Song, Terzis, and Tramer]{carlini2022membership}
N.~Carlini, S.~Chien, M.~Nasr, S.~Song, A.~Terzis, and F.~Tramer.
\newblock Membership inference attacks from first principles, 2022.

\bibitem[Chen et~al.(2020)Chen, Yu, Zhang, and Fritz]{chen2020gan}
D.~Chen, N.~Yu, Y.~Zhang, and M.~Fritz.
\newblock Gan-leaks: A taxonomy of membership inference attacks against generative models.
\newblock In \emph{Proceedings of the 2020 ACM SIGSAC conference on computer and communications security}, pages 343--362, 2020.

\bibitem[Chen et~al.(2022)Chen, Yu, and Fritz]{chen2022relaxloss}
D.~Chen, N.~Yu, and M.~Fritz.
\newblock Relaxloss: Defending membership inference attacks without losing utility.
\newblock In \emph{International Conference on Learning Representations}, 2022.

\bibitem[De et~al.(2022)De, Berrada, Hayes, Smith, and Balle]{de2022unlocking}
S.~De, L.~Berrada, J.~Hayes, S.~L. Smith, and B.~Balle.
\newblock Unlocking high-accuracy differentially private image classification through scale.
\newblock \emph{arXiv preprint arXiv:2204.13650}, 2022.

\bibitem[Deng et~al.(2009)Deng, Dong, Socher, Li, Li, and Fei-Fei]{deng2009imagenet}
J.~Deng, W.~Dong, R.~Socher, L.-J. Li, K.~Li, and L.~Fei-Fei.
\newblock Imagenet: A large-scale hierarchical image database.
\newblock In \emph{2009 IEEE conference on computer vision and pattern recognition}, pages 248--255. Ieee, 2009.

\bibitem[Doerr et~al.(2014)Doerr, Gnewuch, and Wahlstr{\"o}m]{doerr2014calculation}
C.~Doerr, M.~Gnewuch, and M.~Wahlstr{\"o}m.
\newblock Calculation of discrepancy measures and applications.
\newblock \emph{A panorama of discrepancy theory}, pages 621--678, 2014.

\bibitem[Duan et~al.(2023)Duan, Kong, Wang, Shi, and Xu]{duan2023diffusion}
J.~Duan, F.~Kong, S.~Wang, X.~Shi, and K.~Xu.
\newblock Are diffusion models vulnerable to membership inference attacks?
\newblock In \emph{International Conference on Machine Learning}, pages 8717--8730. PMLR, 2023.

\bibitem[Duan et~al.(2024)Duan, Suri, Mireshghallah, Min, Shi, Zettlemoyer, Tsvetkov, Choi, Evans, and Hajishirzi]{duan2024membership}
M.~Duan, A.~Suri, N.~Mireshghallah, S.~Min, W.~Shi, L.~Zettlemoyer, Y.~Tsvetkov, Y.~Choi, D.~Evans, and H.~Hajishirzi.
\newblock Do membership inference attacks work on large language models?
\newblock \emph{arXiv preprint arXiv:2402.07841}, 2024.

\bibitem[Dwork et~al.(2006)Dwork, McSherry, Nissim, and Smith]{dwork2006calibrating}
C.~Dwork, F.~McSherry, K.~Nissim, and A.~Smith.
\newblock Calibrating noise to sensitivity in private data analysis.
\newblock In \emph{Theory of Cryptography: Third Theory of Cryptography Conference, TCC 2006, New York, NY, USA, March 4-7, 2006. Proceedings 3}, pages 265--284. Springer, 2006.

\bibitem[Dwork et~al.(2014)Dwork, Roth, et~al.]{dwork2014algorithmic}
C.~Dwork, A.~Roth, et~al.
\newblock The algorithmic foundations of differential privacy.
\newblock \emph{Foundations and Trends{\textregistered} in Theoretical Computer Science}, 9\penalty0 (3--4):\penalty0 211--407, 2014.

\bibitem[Fredrikson et~al.(2014)Fredrikson, Lantz, Jha, Lin, Page, and Ristenpart]{fredrikson2014privacy}
M.~Fredrikson, E.~Lantz, S.~Jha, S.~Lin, D.~Page, and T.~Ristenpart.
\newblock Privacy in pharmacogenetics: An $\{$End-to-End$\}$ case study of personalized warfarin dosing.
\newblock In \emph{23rd USENIX security symposium (USENIX Security 14)}, pages 17--32, 2014.

\bibitem[Fu et~al.(2023)Fu, Wang, Gao, Liu, Li, and Jiang]{fu2023practical}
W.~Fu, H.~Wang, C.~Gao, G.~Liu, Y.~Li, and T.~Jiang.
\newblock Practical membership inference attacks against fine-tuned large language models via self-prompt calibration.
\newblock \emph{arXiv preprint arXiv:2311.06062}, 2023.

\bibitem[Gong and Liu(2018)]{gong2018attribute}
N.~Z. Gong and B.~Liu.
\newblock Attribute inference attacks in online social networks.
\newblock \emph{ACM Transactions on Privacy and Security (TOPS)}, 21\penalty0 (1):\penalty0 1--30, 2018.

\bibitem[Gottlieb et~al.(2018)Gottlieb, Kaufman, Kontorovich, and Nivasch]{gottlieb2018learning}
L.-A. Gottlieb, E.~Kaufman, A.~Kontorovich, and G.~Nivasch.
\newblock Learning convex polytopes with margin.
\newblock \emph{Advances in neural information processing systems}, 31, 2018.

\bibitem[Haykin(1994)]{haykin1994neural}
S.~Haykin.
\newblock \emph{Neural networks: a comprehensive foundation}.
\newblock Prentice Hall PTR, 1994.

\bibitem[He et~al.(2016)He, Zhang, Ren, and Sun]{resnet}
K.~He, X.~Zhang, S.~Ren, and J.~Sun.
\newblock Deep residual learning for image recognition.
\newblock In \emph{Proceedings of the IEEE conference on computer vision and pattern recognition}, pages 770--778, 2016.

\bibitem[He and Zhang(2021)]{he2021quantifying}
X.~He and Y.~Zhang.
\newblock Quantifying and mitigating privacy risks of contrastive learning.
\newblock In \emph{Proceedings of the 2021 ACM SIGSAC Conference on Computer and Communications Security}, pages 845--863, 2021.

\bibitem[Hintersdorf et~al.(2022)Hintersdorf, Struppek, Brack, Friedrich, Schramowski, and Kersting]{hintersdorf2022does}
D.~Hintersdorf, L.~Struppek, M.~Brack, F.~Friedrich, P.~Schramowski, and K.~Kersting.
\newblock Does clip know my face?
\newblock \emph{arXiv preprint arXiv:2209.07341}, 2022.

\bibitem[Hu et~al.(2022)Hu, Wang, Sun, Wang, and Xue]{hu2022m}
P.~Hu, Z.~Wang, R.~Sun, H.~Wang, and M.~Xue.
\newblock M4i: Multi-modal models membership inference.
\newblock \emph{Advances in Neural Information Processing Systems}, 35:\penalty0 1867--1882, 2022.

\bibitem[Humphries et~al.(2023)Humphries, Oya, Tulloch, Rafuse, Goldberg, Hengartner, and Kerschbaum]{humphries2023investigating}
T.~Humphries, S.~Oya, L.~Tulloch, M.~Rafuse, I.~Goldberg, U.~Hengartner, and F.~Kerschbaum.
\newblock Investigating membership inference attacks under data dependencies.
\newblock In \emph{2023 IEEE 36th Computer Security Foundations Symposium (CSF)}, pages 473--488. IEEE, 2023.

\bibitem[Iyengar et~al.(2019)Iyengar, Near, Song, Thakkar, Thakurta, and Wang]{iyengar2019towards}
R.~Iyengar, J.~P. Near, D.~Song, O.~Thakkar, A.~Thakurta, and L.~Wang.
\newblock Towards practical differentially private convex optimization.
\newblock In \emph{2019 IEEE Symposium on Security and Privacy (SP)}, pages 299--316. IEEE, 2019.

\bibitem[Jayaraman et~al.(2020)Jayaraman, Wang, Knipmeyer, Gu, and Evans]{jayaraman2020revisiting}
B.~Jayaraman, L.~Wang, K.~Knipmeyer, Q.~Gu, and D.~Evans.
\newblock Revisiting membership inference under realistic assumptions.
\newblock \emph{arXiv preprint arXiv:2005.10881}, 2020.

\bibitem[Jia et~al.(2017)Jia, Wang, Zhang, and Gong]{jia2017attriinfer}
J.~Jia, B.~Wang, L.~Zhang, and N.~Z. Gong.
\newblock Attriinfer: Inferring user attributes in online social networks using markov random fields.
\newblock In \emph{Proceedings of the 26th International Conference on World Wide Web}, pages 1561--1569, 2017.

\bibitem[Kantchelian et~al.(2014)Kantchelian, Tschantz, Huang, Bartlett, Joseph, and Tygar]{kantchelian2014large}
A.~Kantchelian, M.~C. Tschantz, L.~Huang, P.~L. Bartlett, A.~D. Joseph, and J.~D. Tygar.
\newblock Large-margin convex polytope machine.
\newblock \emph{Advances in Neural Information Processing Systems}, 27, 2014.

\bibitem[Ko et~al.(2023)Ko, Jin, Wang, and Jia]{ko2023practical}
M.~Ko, M.~Jin, C.~Wang, and R.~Jia.
\newblock Practical membership inference attacks against large-scale multi-modal models: A pilot study.
\newblock In \emph{Proceedings of the IEEE/CVF International Conference on Computer Vision}, pages 4871--4881, 2023.

\bibitem[Krizhevsky(2009)]{cifar}
A.~Krizhevsky.
\newblock Learning multiple layers of features from tiny images.
\newblock Technical report, 2009.

\bibitem[Liu et~al.(2021)Liu, Jia, Qu, and Gong]{liu2021encodermi}
H.~Liu, J.~Jia, W.~Qu, and N.~Z. Gong.
\newblock Encodermi: Membership inference against pre-trained encoders in contrastive learning.
\newblock In \emph{Proceedings of the 2021 ACM SIGSAC Conference on Computer and Communications Security}, pages 2081--2095, 2021.

\bibitem[Mansour et~al.(2009)Mansour, Mohri, and Rostamizadeh]{mansour2009domain}
Y.~Mansour, M.~Mohri, and A.~Rostamizadeh.
\newblock Domain adaptation: Learning bounds and algorithms.
\newblock \emph{arXiv preprint arXiv:0902.3430}, 2009.

\bibitem[Matsumoto et~al.(2023)Matsumoto, Miura, and Yanai]{matsumoto2023membership}
T.~Matsumoto, T.~Miura, and N.~Yanai.
\newblock Membership inference attacks against diffusion models.
\newblock In \emph{2023 IEEE Security and Privacy Workshops (SPW)}, pages 77--83. IEEE, 2023.

\bibitem[Mattern et~al.(2023)Mattern, Mireshghallah, Jin, Sch{\"o}lkopf, Sachan, and Berg-Kirkpatrick]{mattern2023membership}
J.~Mattern, F.~Mireshghallah, Z.~Jin, B.~Sch{\"o}lkopf, M.~Sachan, and T.~Berg-Kirkpatrick.
\newblock Membership inference attacks against language models via neighbourhood comparison.
\newblock \emph{arXiv preprint arXiv:2305.18462}, 2023.

\bibitem[Meehan et~al.(2024)Meehan, Bordes, Vincent, Chaudhuri, and Guo]{meehan2024ssl}
C.~Meehan, F.~Bordes, P.~Vincent, K.~Chaudhuri, and C.~Guo.
\newblock Do ssl models have d{\'e}j{\`a} vu? a case of unintended memorization in self-supervised learning.
\newblock \emph{Advances in Neural Information Processing Systems}, 36, 2024.

\bibitem[Mehnaz et~al.(2022)Mehnaz, Dibbo, De~Viti, Kabir, Brandenburg, Mangard, Li, Bertino, Backes, De~Cristofaro, et~al.]{mehnaz2022your}
S.~Mehnaz, S.~V. Dibbo, R.~De~Viti, E.~Kabir, B.~B. Brandenburg, S.~Mangard, N.~Li, E.~Bertino, M.~Backes, E.~De~Cristofaro, et~al.
\newblock Are your sensitive attributes private? novel model inversion attribute inference attacks on classification models.
\newblock In \emph{31st USENIX Security Symposium (USENIX Security 22)}, pages 4579--4596, 2022.

\bibitem[Mireshghallah et~al.(2022)Mireshghallah, Goyal, Uniyal, Berg-Kirkpatrick, and Shokri]{mireshghallah2022quantifying}
F.~Mireshghallah, K.~Goyal, A.~Uniyal, T.~Berg-Kirkpatrick, and R.~Shokri.
\newblock Quantifying privacy risks of masked language models using membership inference attacks.
\newblock \emph{arXiv preprint arXiv:2203.03929}, 2022.

\bibitem[M{\"u}ller et~al.(2019)M{\"u}ller, Kornblith, and Hinton]{muller2019does}
R.~M{\"u}ller, S.~Kornblith, and G.~E. Hinton.
\newblock When does label smoothing help?
\newblock \emph{Advances in neural information processing systems}, 32, 2019.

\bibitem[Nasr et~al.(2021)Nasr, Songi, Thakurta, Papernot, and Carlin]{nasr2021adversary}
M.~Nasr, S.~Songi, A.~Thakurta, N.~Papernot, and N.~Carlin.
\newblock Adversary instantiation: Lower bounds for differentially private machine learning.
\newblock In \emph{2021 IEEE Symposium on security and privacy (SP)}, pages 866--882. IEEE, 2021.

\bibitem[Niederreiter(1972)]{niederreiter1972discrepancy}
H.~Niederreiter.
\newblock Discrepancy and convex programming.
\newblock \emph{Annali di matematica pura ed applicata}, 93:\penalty0 89--97, 1972.

\bibitem[Papernot et~al.(2016)Papernot, McDaniel, Sinha, and Wellman]{papernot2016towards}
N.~Papernot, P.~McDaniel, A.~Sinha, and M.~Wellman.
\newblock Towards the science of security and privacy in machine learning.
\newblock \emph{arXiv preprint arXiv:1611.03814}, 2016.

\bibitem[Paszke et~al.(2017)Paszke, Gross, Chintala, Chanan, Yang, DeVito, Lin, Desmaison, Antiga, and Lerer]{paszke2017PyTorch}
A.~Paszke, S.~Gross, S.~Chintala, G.~Chanan, E.~Yang, Z.~DeVito, Z.~Lin, A.~Desmaison, L.~Antiga, and A.~Lerer.
\newblock Automatic differentiation in pytorch.
\newblock In \emph{NIPS-W}, 2017.

\bibitem[Paszke et~al.(2019)Paszke, Gross, Massa, Lerer, Bradbury, Chanan, Killeen, Lin, Gimelshein, Antiga, et~al.]{paszke2019pytorch}
A.~Paszke, S.~Gross, F.~Massa, A.~Lerer, J.~Bradbury, G.~Chanan, T.~Killeen, Z.~Lin, N.~Gimelshein, L.~Antiga, et~al.
\newblock Pytorch: An imperative style, high-performance deep learning library.
\newblock \emph{Advances in neural information processing systems}, 32, 2019.

\bibitem[Radford et~al.(2019)Radford, Wu, Child, Luan, Amodei, and Sutskever]{radford2019language}
A.~Radford, J.~Wu, R.~Child, D.~Luan, D.~Amodei, and I.~Sutskever.
\newblock Language models are unsupervised multitask learners.
\newblock 2019.

\bibitem[Rezaei and Liu(2021)]{rezaei2021difficulty}
S.~Rezaei and X.~Liu.
\newblock On the difficulty of membership inference attacks.
\newblock In \emph{Proceedings of the IEEE/CVF Conference on Computer Vision and Pattern Recognition}, pages 7892--7900, 2021.

\bibitem[Shokri et~al.(2017)Shokri, Stronati, Song, and Shmatikov]{shokri2017membership}
R.~Shokri, M.~Stronati, C.~Song, and V.~Shmatikov.
\newblock Membership inference attacks against machine learning models.
\newblock In \emph{2017 IEEE symposium on security and privacy (SP)}, pages 3--18. IEEE, 2017.

\bibitem[Song and Mittal(2021)]{song2021systematic}
L.~Song and P.~Mittal.
\newblock Systematic evaluation of privacy risks of machine learning models.
\newblock In \emph{30th USENIX Security Symposium (USENIX Security 21)}, pages 2615--2632, 2021.

\bibitem[Song and Marn()]{song2020introducing}
S.~Song and D.~Marn.
\newblock Introducing a new privacy testing library in tensorflow (2020).
\newblock \emph{URL https://blog. tensorflow. org/2020/06/introducing-new-privacy-testing-library. html}.

\bibitem[Touvron et~al.(2023)Touvron, Martin, Stone, Albert, Almahairi, Babaei, Bashlykov, Batra, Bhargava, Bhosale, et~al.]{touvron2023llama}
H.~Touvron, L.~Martin, K.~Stone, P.~Albert, A.~Almahairi, Y.~Babaei, N.~Bashlykov, S.~Batra, P.~Bhargava, S.~Bhosale, et~al.
\newblock Llama 2: Open foundation and fine-tuned chat models.
\newblock \emph{arXiv preprint arXiv:2307.09288}, 2023.

\bibitem[Watson et~al.(2021)Watson, Guo, Cormode, and Sablayrolles]{watson2021importance}
L.~Watson, C.~Guo, G.~Cormode, and A.~Sablayrolles.
\newblock On the importance of difficulty calibration in membership inference attacks.
\newblock In \emph{International Conference on Learning Representations}, 2021.

\bibitem[Wen et~al.(2022)Wen, Bansal, Kazemi, Borgnia, Goldblum, Geiping, and Goldstein]{wen2022canary}
Y.~Wen, A.~Bansal, H.~Kazemi, E.~Borgnia, M.~Goldblum, J.~Geiping, and T.~Goldstein.
\newblock Canary in a coalmine: Better membership inference with ensembled adversarial queries.
\newblock \emph{arXiv preprint arXiv:2210.10750}, 2022.

\bibitem[Wu et~al.(2023)Wu, Zhou, Weinberger, and Guo]{wu2023does}
R.~Wu, J.~P. Zhou, K.~Q. Weinberger, and C.~Guo.
\newblock Does label differential privacy prevent label inference attacks?
\newblock In \emph{International Conference on Artificial Intelligence and Statistics}, pages 4336--4347. PMLR, 2023.

\bibitem[Wu et~al.(2022)Wu, Yu, Li, Backes, and Zhang]{wu2022membership}
Y.~Wu, N.~Yu, Z.~Li, M.~Backes, and Y.~Zhang.
\newblock Membership inference attacks against text-to-image generation models.
\newblock 2022.

\bibitem[Yeom et~al.(2018)Yeom, Giacomelli, Fredrikson, and Jha]{yeom2018privacy}
S.~Yeom, I.~Giacomelli, M.~Fredrikson, and S.~Jha.
\newblock Privacy risk in machine learning: Analyzing the connection to overfitting, 2018.

\bibitem[Yu et~al.(2021)Yu, Naik, Backurs, Gopi, Inan, Kamath, Kulkarni, Lee, Manoel, Wutschitz, et~al.]{yu2021differentially}
D.~Yu, S.~Naik, A.~Backurs, S.~Gopi, H.~A. Inan, G.~Kamath, J.~Kulkarni, Y.~T. Lee, A.~Manoel, L.~Wutschitz, et~al.
\newblock Differentially private fine-tuning of language models.
\newblock In \emph{International Conference on Learning Representations}, 2021.

\bibitem[Yun et~al.(2019)Yun, Han, Oh, Chun, Choe, and Yoo]{yun2019cutmix}
S.~Yun, D.~Han, S.~J. Oh, S.~Chun, J.~Choe, and Y.~Yoo.
\newblock Cutmix: Regularization strategy to train strong classifiers with localizable features.
\newblock In \emph{International Conference on Computer Vision (ICCV)}, 2019.

\bibitem[Zhang et~al.(2018)Zhang, Cisse, Dauphin, and Lopez-Paz]{zhang2018mixup}
H.~Zhang, M.~Cisse, Y.~N. Dauphin, and D.~Lopez-Paz.
\newblock mixup: Beyond empirical risk minimization, 2018.

\bibitem[Zhang et~al.(2020)Zhang, Jia, Pei, Wang, Li, and Song]{zhang2020secret}
Y.~Zhang, R.~Jia, H.~Pei, W.~Wang, B.~Li, and D.~Song.
\newblock The secret revealer: Generative model-inversion attacks against deep neural networks.
\newblock In \emph{Proceedings of the IEEE/CVF conference on computer vision and pattern recognition}, pages 253--261, 2020.

\bibitem[Zhong et~al.(2020)Zhong, Zheng, Kang, Li, and Yang]{zhong2020random}
Z.~Zhong, L.~Zheng, G.~Kang, S.~Li, and Y.~Yang.
\newblock Random erasing data augmentation.
\newblock In \emph{Proceedings of the AAAI conference on artificial intelligence}, volume~34, pages 13001--13008, 2020.

\bibitem[Zhu et~al.(2019)Zhu, Liu, and Han]{zhu2019deep}
L.~Zhu, Z.~Liu, and S.~Han.
\newblock Deep leakage from gradients.
\newblock \emph{Advances in neural information processing systems}, 32, 2019.

\bibitem[Zhu et~al.(2024)Zhu, Guo, Feng, and Simeone]{zhu2024uncertainty}
M.~Zhu, C.~Guo, C.~Feng, and O.~Simeone.
\newblock Uncertainty, calibration, and membership inference attacks: An information-theoretic perspective.
\newblock \emph{arXiv preprint arXiv:2402.10686}, 2024.

\bibitem[Ziegler(2012)]{ziegler2012lectures}
G.~M. Ziegler.
\newblock \emph{Lectures on polytopes}, volume 152.
\newblock Springer Science \& Business Media, 2012.

\end{thebibliography}


%
\newpage
\appendix
\section{Proofs of Theorems in Section~\ref{sec:discrepancy}} \label{sec:appendix}
\subsection{Proof of Theorem~\ref{thm:cvx_dis_up}}
\begin{proof}[Proof of Theorem~\ref{thm:cvx_dis_up}.]
We start from some notations. Given a discriminative set $Q\subseteq \bbR^{2C}$, we define an MIA $m_Q(f, z):=\mathds{1}[(f(x), y)\in Q]$.

The proof for $m_{msp, \tau}$ and $m_{ent, \tau}$ is straightforward because the MSP score and the ENT score are convex/concave in $(f(x), y)$, hence thresholding the convex/concave function gives a convex/concave discriminative set.

For $m_{\rm CE}$, it takes more effort, because $\sum_{c\in [C]}-y_c\log(f(x)_c)$ is no longer convex in $(f(x), y)$. 
We will construct a $\tilde{m}(f, z)=\mathds{1}\left[g(f, z)< \tau\right]$ for $(f(x), y)\in (0, 1]^C \times [0, 1]^C$, which has the following property:
\begin{enumerate}
\item $g(f, z)= \sum_{c\in [C]}-y_c\log(f(x)_c)$, for $(f(x), y)$ in $(0,1]^C \times {\color{red}\{}0,1{\color{red}\}}^C$.
\item $g(f, z)$ is a convex function of $(f(x), y)$ on $(0, 1]^C \times [0, 1]^C$.
\end{enumerate}
If such $\tilde{m}$ exists, because of the second property, by following the same argument for $m_{\rm MSP}$, we can prove there exists a convex set $Q\subseteq \bbR^{2C}$ s.t. $m_Q(f, z)=\tilde{m}(f, z)$. 
Moreover, property 1 implies that $m(f, z)=\tilde{m}(f, z)$ for all $(f(x), y)\in\{(f(x), y)| (x, y) \in S\cup D_{\rm test}\}$, which completes the proof.

Now we are going to show the construction of $\tilde{m}(f, z)=\mathds{1}\left[g(f, z)< \tau\right]$. 
We define
$$
g(f, z):=\sum_{c\in [C]}-y_c\log(f(x)_c) + \sum_{c: y_c\neq 0} y_c\log(y_c), \forall(f(x), y)\in (0, 1]^C \times [0, 1]^C.
$$
By definition, property 1 naturally holds. We are going to verify the property 2.
Firstly, $g(f, z)$ is a continuous function on $(0, 1]^C \times [0, 1]^C$ by the fact that $\lim_{a\to 0}a\log(a)=0$.
Secondly, we can prove $\nabla^2_{f(x), y}g(f, z)\succeq \mathbf{0}\ \forall(f(x), y)\in (0, 1]^C \times {\color{red}(}0, 1]^C$.
It is sufficient to prove $\nabla^2_{f(x)_c, y_c}g(f, z)\succeq \mathbf{0}$ because $\nabla^2_{f(x)_c, y_{c'}}g(f, z)= \mathbf{0}$ when $c\neq c'$.
To see $\nabla^2_{f(x)_c, y_c}g(f, z)\succeq \mathbf{0}$, $\forall \ba\in \bbR^2$,
$$
\ba\nabla^2_{f(x)_c, y_c}g(f, z)\ba^{\top} = \frac{y_c}{f(x)_c^2}\ba_1^2 - \frac{2}{f(x)_c}\ba_1\ba_2 + \frac{1}{y_c}\ba_2^2=\frac{1}{y_c} \left(\frac{y_c}{f(x)_c}\ba_1-\ba_2\right)^2\geq 0.
$$
Therefore $g(f, z)$ is a convex function of $(f(x), y)$ on $(0, 1]^C \times [0, 1]^C$.

Similarly, for $m_{\rm ME}$, we are going to construct an $F$ such that 
\begin{enumerate}
\item $g(f, z)=-\sum_{c\in [C]}\left((1-f(x)_c)\log(f(x)_c)y_c + f(x)_{c}\log(1-f(x)_{c})(1-y_c)\right)$ for $(f(x), y)$ in $(0,1]^C \times \{0,1\}^C$.
\item $g(f, z)$ is a convex function of $(f(x), y)$ on $(0, 1]^C \times [0, 1]^C$.
\end{enumerate}
We define 
$$g(f, z):=-\sum_{c\in [C]}\left((1-f(x)_c)\log(f(x)_c)y_c + f(x)_{c}\log(1-f(x)_{c})(1-y_c)\right)
$$
$$
+5\sum_{c: y_c\neq 0} y_c\log(y_c) + 5\sum_{c: y_c\neq 1} (1-y_c)\log(1-y_c).
$$

By definition, property 1 naturally holds. We are going to verify the property 2.
Firstly, $g(f, z)$ is a continuous function on $(0, 1]^C \times [0, 1]^C$ by the fact that $\lim_{a\to 0}a\log(a)=0$.
Secondly, we can prove $\nabla^2_{f(x), y}g(f, z)\succeq \mathbf{0}\ \forall(f(x), y)\in (0, 1]^C \times {\color{red}(}0, 1]^C$.
It is sufficient to prove $\nabla^2_{f(x)_c, y_c}g(f, z)\succeq \mathbf{0}$ because $\nabla^2_{f(x)_c, y_{c'}}g(f, z)= \mathbf{0}$ when $c\neq c'$.
To see $\nabla^2_{f(x)_c, y_c}g(f, z)\succeq \mathbf{0}$, $\forall \ba\in \bbR^2$,

\begin{align*}
&\ba\nabla^2_{f(x)_c, y_c}g(f, z)\ba^{\top} \\
&= \left(y_c\cdot \left(\frac{1}{f(x)_c^2} + \frac{1}{f(x)_c}\right) + (1-y_c)\cdot \left(\frac{1}{(1-f(x)_c)^2} + \frac{1}{1-f(x)_c}\right) \right)\ba_1^2\\
&+ 2 \left(\frac{1}{1-f(x)_c}-\frac{1}{f(x)_c}+\log(f(x)_c)-\log(1-f(x)_c)\right)\ba_1\ba_2+ \left(\frac{3}{y_c} + \frac{3}{1-y_c}\right)\ba_2^2\\
&= \frac{1}{y_c}\left(\frac{y_c}{f(x)_c}\ba_1 - \ba_2\right)^2 
+ \frac{1}{1-y_c}\left(\frac{1-y_c}{1-f(x)_c}\ba_1 - \ba_2\right)^2\\
&+\left(\frac{y_c}{f(x)_c}\ba_1^2 + 2\log(f(x)_c)\ba_1\ba_2 + \frac{4}{y_c}\ba_2^2\right)+\left(\frac{1-y_c}{1-f(x)_c}\ba_1^2 + 2\log(1-f(x)_c)\ba_1\ba_2 + \frac{4}{1-y_c}\ba_2^2\right)\\
&\geq \left(\frac{y_c}{f(x)_c}\ba_1^2 + 2\log(f(x)_c)\ba_1\ba_2 + \frac{4}{y_c}\ba_2^2\right)+\left(\frac{1-y_c}{1-f(x)_c}\ba_1^2 + 2\log(1-f(x)_c)\ba_1\ba_2 + \frac{4}{1-y_c}\ba_2^2\right)
\end{align*}
We are going to first prove $\frac{y_c}{f(x)_c}\ba_1^2 + 2\log(f(x)_c)\ba_1\ba_2 + \frac{4}{y_c}\ba_2^2\geq 0 $. Because $\frac{y_c}{f(x)_c}>0$, it is sufficient to prove
$(2\log(f(x)_c))^2\leq 4\cdot \frac{y_c}{f(x)_c}\cdot \frac{4}{y_c}$, which is equivalent to $-\log(f(x)_c)\leq 2\sqrt{\frac{1}{f(x)_c}}$. Define $h(f(x)_c)=2\sqrt{\frac{1}{f(x)_c}} +\log(f(x)_c)$. $\forall f(x)_c\in(0, 1]$, $h'(f(x)_c)=-\frac{1}{\sqrt{f(x)_c}f(x)_c} +\frac{1}{f(x)_c}<0$. Thus, $h(f(x)_c)\geq h(1)=\sqrt{2}>0$ and $\frac{y_c}{f(x)_c}\ba_1^2 + 2\log(f(x)_c)\ba_1\ba_2 + \frac{4}{y_c}\ba_2^2\geq 0 $ has been proved. 

Similarly, we can prove $\left(\frac{1-y_c}{1-f(x)_c}\ba_1^2 + 2\log(1-f(x)_c)\ba_1\ba_2 + \frac{4}{1-y_c}\ba_2^2\right)\geq 0$ and we have completed the proof for $\ba\nabla^2_{f(x)_c, y_c}g(f, z)\ba^{\top}\geq 0$. We now have the convexity of $g(f, z)$  on $(0, 1]^C \times [0, 1]^C.$
\end{proof}

\subsection{Proof of Proposition~\ref{thm:cvx_up2}}
\begin{proof}[Proof of Proposition~\ref{thm:cvx_up2}]
Our proof mostly follows the proof in \citet{niederreiter1972discrepancy}. Suppose $z=(x, y)$. $\forall Q\in \calQ_{cvx}'$, we are going to find $Q_1, Q_2\in \calQ_{cvx, k}'$ such that $\bbP_{z\sim \calD}((f(x), y)\in Q_1)\leq \bbP_{z\sim \calD}((f(x), y)\in Q)\leq \bbP_{z\sim \calD}((f(x), y)\in Q_2)$ and $\bbP_{z\sim S}((f(x), y)\in Q)=\bbP_{z\sim S}((f(x), y)\in Q_1)=\bbP_{z\sim S}((f(x), y)\in Q_2)$.
Within this $Q_1, Q_2$, 
$$D(S, \calD|Q)\leq \max\{D(S, \calD|Q_1), D(S, \calD|Q_2)\},$$
and therefore $D_{\calQ_{cvx}'}(S, \calD)\leq D_{\calQ_{cvx, k}'}(S, \calD)$.
It is obvious that $D_{\calQ_{cvx}'}(S, \calD)\geq D_{\calQ_{cvx, k}'}(S, \calD)$ because $\calQ_{cvx, k}'\subseteq \calQ_{cvx}'$.
Thus $D_{\calQ_{cvx}'}(S, \calD) = D_{\calQ_{cvx, k}'}(S, \calD)$.

To find $Q_1$, for any convex set $Q\in Q_{cvx}$, 
	if $Q\cap S=\emptyset$, we simply choose $Q_1=\emptyset$.
	If $Q\cap S\neq\emptyset$, we can consider a convex hull $Q_1$ for $Q\cap S$.
	From the definition of the convex hull, $Q_1\subseteq Q$.
	In both two cases above,	
	\begin{align*}
		\bbP_{z\sim S}((f(x), y)\in Q) - \bbP_{z\sim \calD}((f(x), y)\in Q) \leq \bbP_{z\sim S}((f(x), y)\in Q_1) - \bbP_{z\sim \calD}((f(x), y)\in Q_1)
	\end{align*}
	Because $Q\cap S$ is a discrete point set, the convex hull $Q_1$ would be a convex polytope whose vertices are a subset of $Q\cap S$.
	The Upper Bound Theorem~\citep{ziegler2012lectures} shows that the number of facets of a convex polytope with at most $|S|$ vertices can be bounded by $\binom{|S|}{2C}$, where $2C$ is the dimensionality of the space.
	
To find $Q_2$, we follow the proof in \citet{niederreiter1972discrepancy} to construct a $Q'$ first: 
Because $Q$ is a closed set by assumption, for each $z\in S\backslash (Q\cap S)$, we can find a supporting hyperplane of $Q$ such that $Q$ lies in the closed halfspace $H_z$ defined by this supporting hyperplane.
Then we can define $Q_2:=\cap_{z\in S\backslash (Q\cap S)}H_z$.
Obviously, $Q\subseteq Q_2$ by the definition of $H_z$.
Therefore, $\bbP_{z\sim \calD}((f(x), y)\in Q)\leq \bbP_{z\sim \calD}((f(x), y)\in Q_2)$.
Moreover, $\forall z\in S\backslash (Q\cap S)$, $z\notin Q_2$; $\forall z\in Q\cap S$, $z\in Q\subseteq Q_2$.
Thus, $\bbP_{z\sim S}((f(x), y)\in Q) = \bbP_{z\sim S}((f(x), y)\in Q_2)$.
Lastly, by the definition $Q_2$, it is a convex polytope with at most $|S|\leq \binom{|S|}{2C}$ facets.


\end{proof}

\section{Additional Experiment Results}
\label{sec:app_exp}
\paragraph{CPM setup.} For the models trained on CIFA-10, CIFAR-100, Purchase100 and Texas100, we get our new metric CPM by optimizing the objective in Equation~\ref{eq:cpm} and pick the number of facets $K = 1000$, lr $= 0.1, 0.01, 0.001$ and batch size 10000. For ImageNet pre-trained models, we get CPM with $K = 1000$, lr $= 0.001$ and batch size = 512. The optimizer is Adam. We select hyper-parameters that give the minimum optimization loss. All optimizations are conducted with GPU on NVIDIA GeForce RTX 3080, and the longest time for one optimization is 20 minutes. 

\paragraph{Effect of $K$ for approximation quality of CPM on additional datasets.} The parameter $K$ in the CPM computation measures approximation quality. We plot the advantage of CPM versus the number of facets $K$ for other three datasets CIFAR-10 (Figure~\ref{fig:ablation_cifar10}), Texas (Figure~\ref{fig:ablation_Texas}), Purchase (Figure~\ref{fig:ablation_Purchase}). We see the similar behavior to the figure of CIFAR-100 (Figure~\ref{fig:ablation}) in the main paper, that is, this is a monotone increasing function where larger $K$ has higher advantage, suggesting that CPM is approaching the true upper bound CPB well when $K$ is getting larger. Furthermore, for MixUp and RelaxLoss models, we see that the CPM outperforms the advantage of other scores even at $K = 100$, which suggests that in these cases, even a reasonably large value of $K$ can achieve larger advantage.

\label{sec:app_exp}
\begin{figure}[t!]
    \centering
    \subfigure[Vanilla]{\includegraphics[width=0.3\textwidth]{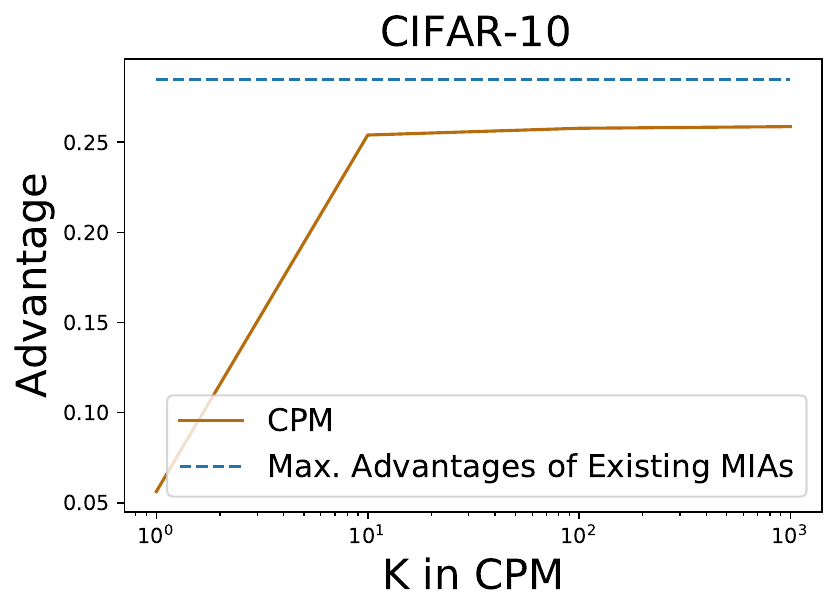}\label{fig:vanilla}}
    \hfill
    \subfigure[MixUp]{\includegraphics[width=0.3\textwidth]{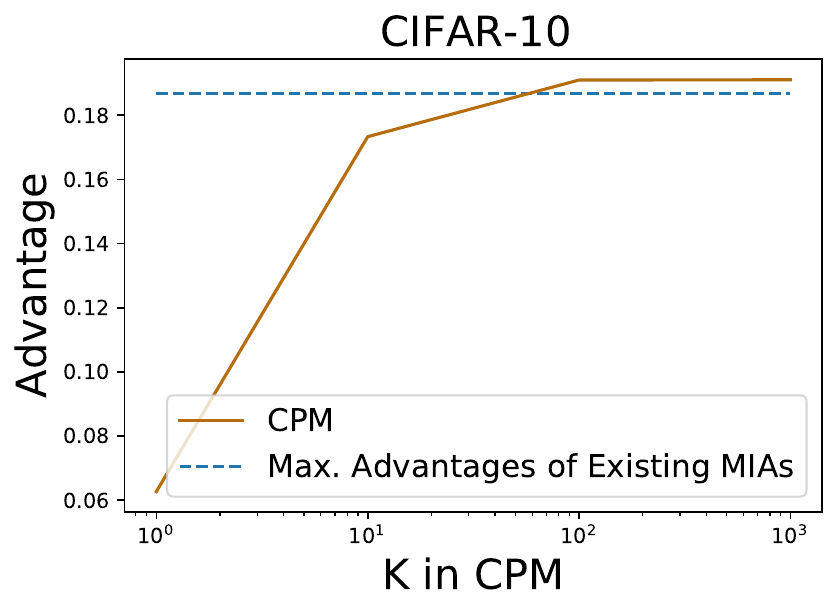}\label{fig:mixup}}
    \hfill
    \subfigure[Relaxloss]{\includegraphics[width=0.3\textwidth]{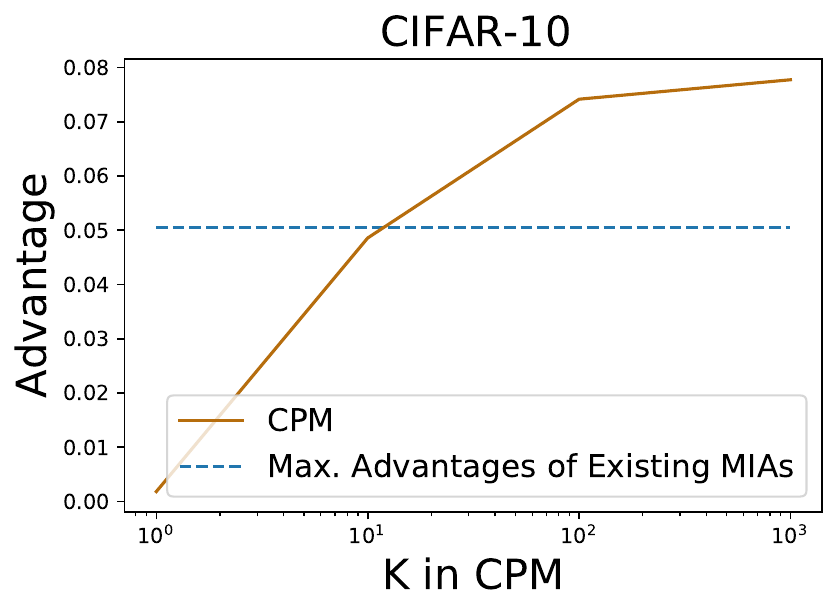}\label{fig:relaxloss}}
    \hfill
    \captionof{figure}{The advantage on CIFAR-10 when we find CPM with different numbers of facets $K$.}
   \label{fig:ablation_cifar10}
\end{figure}
\begin{figure}[t!]
    \centering
    \subfigure[Vanilla]{\includegraphics[width=0.3\textwidth]{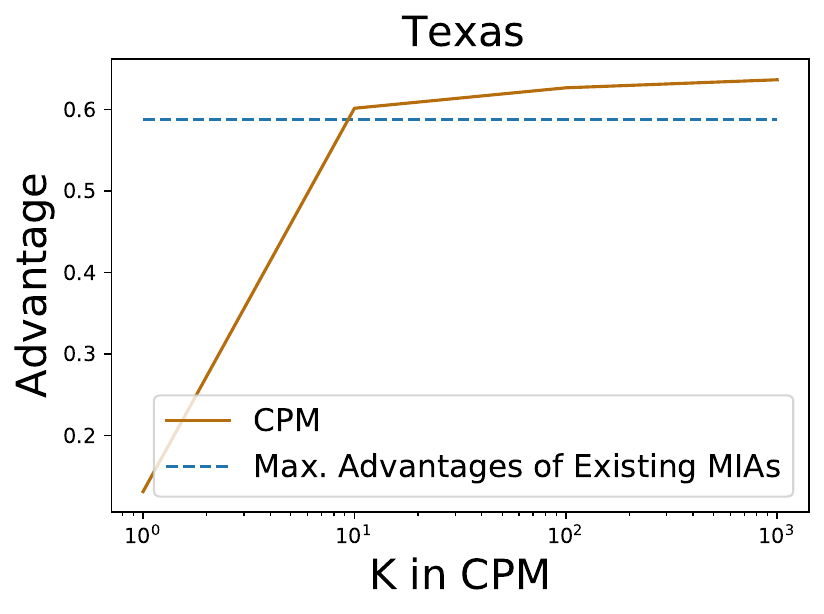}\label{fig:vanilla_abla_Texas}}
    \hfill
    \subfigure[MixUp]{\includegraphics[width=0.3\textwidth]{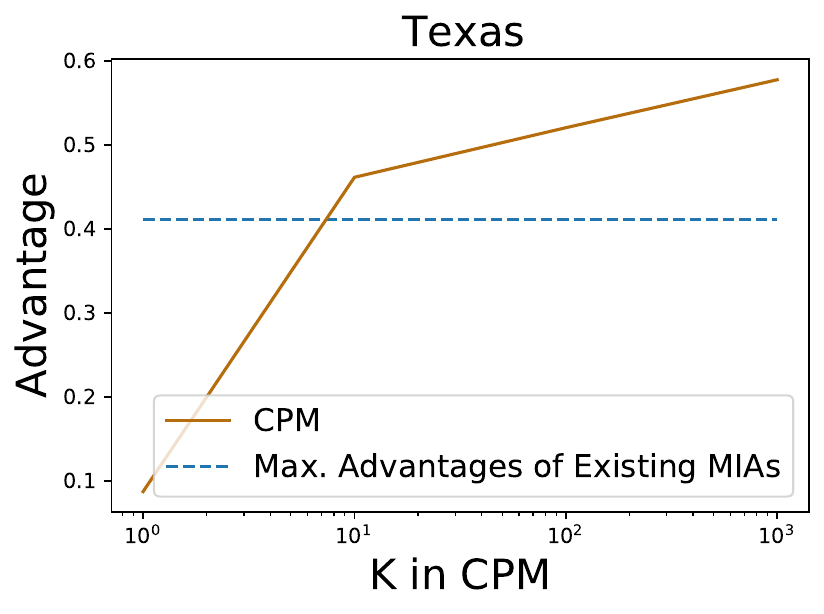}\label{fig:mixup_abla_Texas}}
    \hfill
    \subfigure[Relaxloss]{\includegraphics[width=0.3\textwidth]{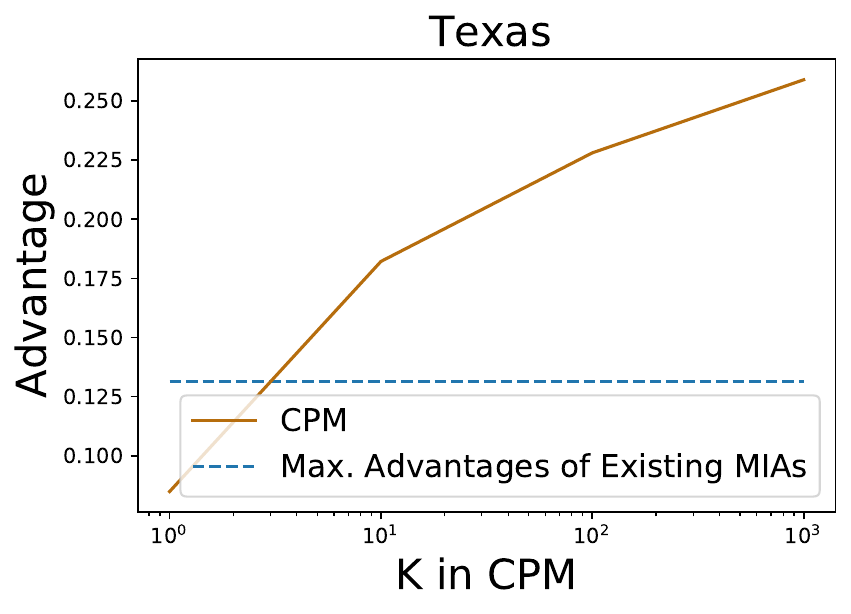}\label{fig:relaxloss_abla_Texas}}
    \hfill
   \captionof{figure}{The advantage on Texas when we find CPM with different numbers of facets $K$.}
   \label{fig:ablation_Texas}
\end{figure}
\begin{figure}[t!]
    \centering
    \subfigure[Vanilla]{\includegraphics[width=0.3\textwidth]{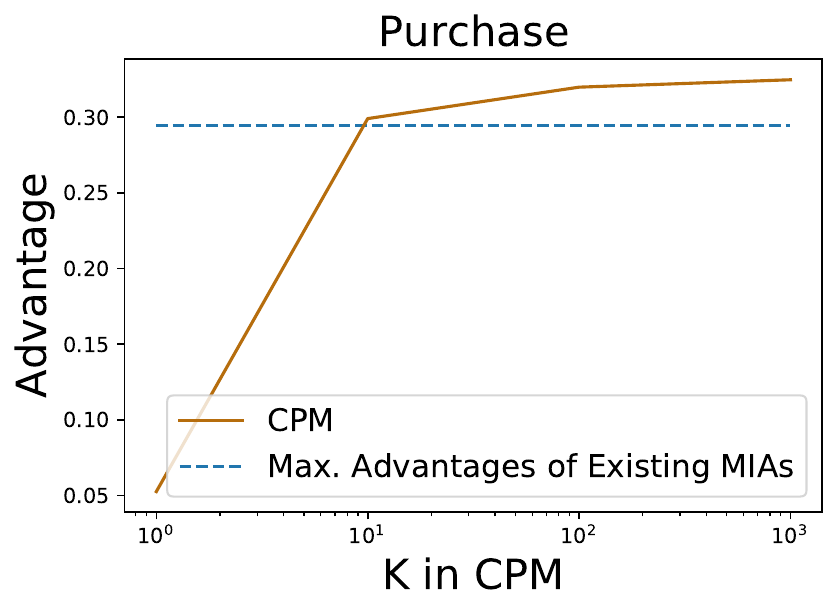}\label{fig:vanilla_abla_Purchase}}
    \hfill
    \subfigure[MixUp]{\includegraphics[width=0.3\textwidth]{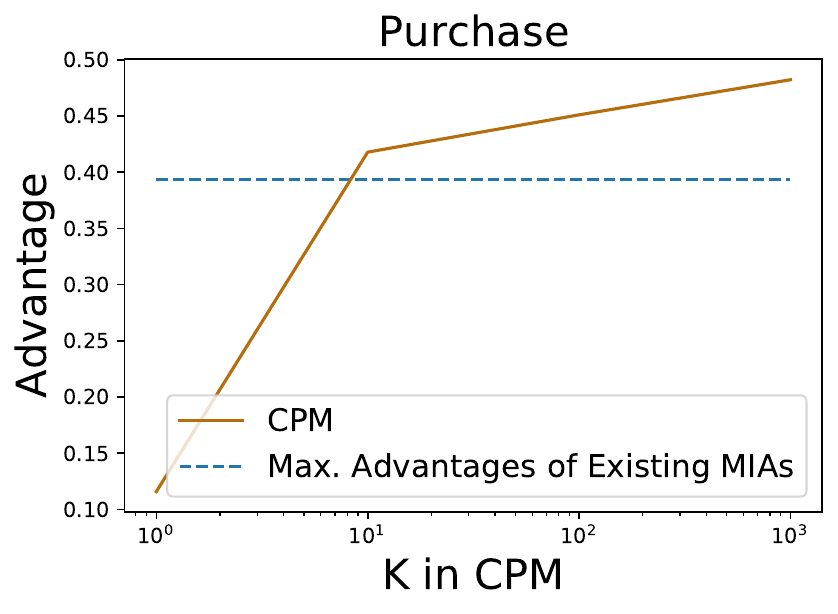}\label{fig:mixup_abla_Purchase}}
    \hfill
    \subfigure[Relaxloss]{\includegraphics[width=0.3\textwidth]{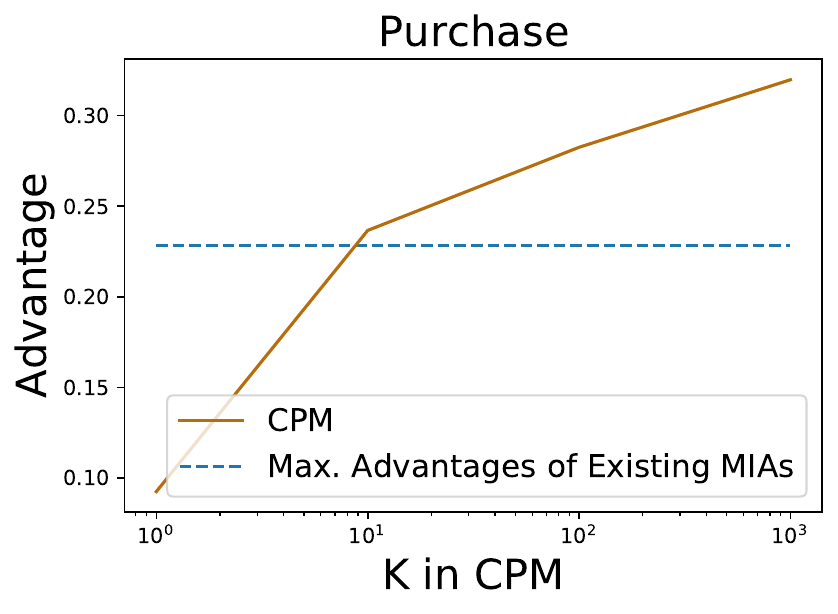}\label{fig:relaxloss_abla_Purchase}}
    \hfill
    \captionof{figure}{The advantage on Purchase when we find CPM with different numbers of facets $K$.}
   \label{fig:ablation_Purchase}
\end{figure}
\end{document}